\documentclass[]{sn-jnl}
\usepackage{graphicx}%
\usepackage{multirow}%
\usepackage{amsmath,amssymb,amsfonts}%
\usepackage{amsthm}%
\usepackage{mathrsfs}%
\usepackage[title]{appendix}%
\usepackage{xcolor}%
\usepackage{textcomp}%
\usepackage{manyfoot}%
\usepackage{booktabs}%
\usepackage{algorithm}%
\usepackage{algorithmicx}%
\usepackage{algpseudocode}%
\usepackage{listings}%
\usepackage{hyperref}
\usepackage{subfigure}

\usepackage{geometry}
\usepackage{tikz}
\usepackage{caption}
\usepackage{adjustbox}
\usetikzlibrary{shapes.geometric, arrows}
\tikzstyle{startstop} = [rectangle, rounded corners, 
minimum width=3cm, 
minimum height=1cm,
text centered, 
draw=black, 
fill=red!30]

\tikzstyle{io} = [trapezium, 
trapezium stretches=true, 
trapezium left angle=70, 
trapezium right angle=110, 
minimum width=3cm, 
minimum height=1cm, text centered, 
draw=black, fill=blue!30]

\tikzstyle{process} = [rectangle, 
minimum width=3cm, 
minimum height=1cm, 
text centered, 
text width=3cm, 
draw=black, 
fill=orange!30]

\tikzstyle{decision} = [diamond, 
minimum width=3cm, 
minimum height=2cm,
text centered, 
draw=black, 
fill=green!30]
\tikzstyle{arrow} = [thick,->,>=stealth]
\usepackage[style=apa, backend=biber]{biblatex}
\theoremstyle{thmstyleone}%

\theoremstyle{thmstyletwo}%
\newtheorem{lemma}{Lemma}

\theoremstyle{thmstylethree}%

\begin{document}

\title[Article Title]{An exact coverage path planning algorithm for UAV-based search and rescue operations}

\author[a]{\fnm{Sina} \sur{Kazemdehbashi}}

\author[b]{\fnm{Yanchao} \sur{Liu}}

\affil[a]{\orgdiv{hm3369@wayne.edu}}
\affil[b]{\orgdiv{yanchaoliu@wayne.edu}}

\abstract{

Unmanned aerial vehicles (UAVs) are increasingly utilized in global search and rescue efforts, enhancing operational efficiency. In these missions, a coordinated swarm of UAVs is deployed to efficiently cover expansive areas by capturing and analyzing aerial imagery and footage. Rapid coverage is paramount in these scenarios, as swift discovery can mean the difference between life and death for those in peril.
This paper focuses on optimizing flight path planning for multiple UAVs in windy conditions to efficiently cover rectangular search areas in minimal time. We address this challenge by dividing the search area into a grid network and formulating it as a mixed-integer program (MIP). Our research introduces a precise lower bound for the objective function and an exact algorithm capable of finding either the optimal solution or a near-optimal solution with a constant absolute gap to optimality. Notably, as the problem complexity increases, our solution exhibits a diminishing relative optimality gap while maintaining negligible computational costs compared to the MIP approach.
}


\maketitle

\section{Introduction}\label{sec:intro}
Thousands of people are reported missing and later found dead each year, due to being trapped or immobilized in harsh environments. For instance, approximately $4,000$ casualties in maritime environments have been reported each year since 2014 (\cite{Cho2021}). To enhance the efficiency of life-saving operations, search and rescue (SAR) teams nowadays try to embrace cutting-edge technologies like UAVs and artificial intelligence to improve safety and operational efficiency (\cite{martinez2021search}). Remote controlled or autopiloted UAVs are quicker to deploy, less expensive to operate and maintain and can reach more locations, compared to other means of search such as ones conducted by canine, on foot or by helicopter. UAVs are becoming the preferred equipment for many SAR missions (\cite{lyu2023unmanned}).

Despite the many benefits, using UAVs to conduct large area searches comes with several unique challenges, such as limitations and uncertainties brought about by the battery capacity, weather conditions, and collision avoidance constraints. The energy consumption, flight time and stability of small-sized multirotor UAVs are particularly susceptible to the effects of wind (\cite{Gianfelice2022}). Furthermore, when multiple UAVs are employed in a SAR mission, the division of tasks and the planning of flight paths in the midst of the above constraints become nontrivial, and often require sophisticated treatment calculation. The operations research that addresses the such challenges are termed \emph{Coverage Path Planning (CPP)} and there is a rich literature about it (\cite{maza2007multiple,Bouzid2017,forsmo2013optimal,di2016coverage}). However, most of the studies ignored weather conditions. We argue that minimizing the overall time to cover the search area should be the primary objective in SAR type of operations, and the relation between flight and wind directions plays an important role in optimizing that objective. This argument is corroborated in \cite{Coombes2018}, which presented a wind-aware area survey method to optimally cover an area with a single UAV. In this paper, we aim to fill a research gap by solving a practical version of the CPP problem that involves using multiple UAVs to cover a rectangular area with explicit consideration for wind conditions. Specifically, we propose to discretize the search area into a grid of square cells whereas the dimension of a cell matches the flight altitude and camera aperture configurations of the mission. The orientation of this artificial grid is set in a way such that UAVs traversing the grid in a Von Neumann fashion (to be defined in Sect. \ref{gridphase}) will fly in a direction either parallel to (same or opposite) or perpendicular to the wind direction. On top of this grid, we formulate the multiple-UAV CPP problem as a mixed-integer programming model which can be solved using commercial solvers such as CPLEX and GUROBI. 
To significantly shorten the computing time for larger instances, we develop a specialized algorithm that can guarantee a feasible solution with a provable performance bound. More specifically, the solution produced by our proposed algorithm either yields the minimal coverage time or yields a coverage time which is exactly $T_p$ longer than the minimum, where $T_p$ is the time it takes for a UAV to traverse one cell in a perpendicular direction to the wind. 

The organization of the paper is as follows: Section \ref{sec: literature-review} reviews the related literature. Section \ref{sec:problem description} defines the problem and presents a MIP formulation of the problem. In Sect. \ref{sec: solution approach}, we derive a mathematical formula to obtain the lower bound of the problem's objective, and introduce an efficient algorithm to construct a feasible solution whose objective is either the same as or close to the lower bound, thereby solving the problem much faster than a commercial MIP solver. 
 Section \ref{sec:Moore connectivity proof} proves that the proposed method can be extended to cases utilizing the Moore neighborhood connectivity (to be defined in Sect. \ref{gridphase}). Validation experiments and sensitivity analysis are presented in Sect. \ref{sec:experimentalRes}. Finally, Section \ref{sec:conclusion} concludes the paper and suggests several extensions and new developments for future work.

\section{Literature Review}\label{sec: literature-review}

The primary objective of the CPP is to develop a path plan to cover the entire target region while avoiding any obstacles.
\textcite{cabreira2019survey} reviewed studies on CPP using UAVs with several classification criteria, including the shape of the search area, how the search area is decomposed into smaller, more manageable parts, the performance metrics being used to guide the path planning, and how the search patterns are formed. Among these criteria, the area decomposition method is the primary factor that determines the nature of the mathematical model and solution approach for the problem. Therefore, we categorize CPP methods into two types based on this factor: ones that employed an exact cellular decomposition of the search area, and ones that employed an approximate (grid-based) decomposition of the search area.

In the exact cell decomposition method, the area of interest is divided into several sub-areas using linear boundaries, and the sub-areas are called cells. \textcite{Latombe1991} presented a trapezoidal decomposition that divided the target region into trapezoidal cells, whereas each cell could be covered with simple back-and-forth motions. \textcite{Choset2000} proposed an improved method called the boustrophedon decomposition that attempted to merge the trapezoidal cells into larger non-convex cells to create opportunities for path length reduction when the robots engaged in back-and-forth motions to cover the cells. \textcite{kong2006distributed} utilized the boustrophedon cell decomposition to address the multi-robot coverage problem. 
\textcite{Jiao2010} presented a method for the CPP problem in polygonal areas.  They defined the width of a convex polygonal shape and demonstrated that a UAV should fly along the vertical direction of the width to achieve a coverage path with the least number of turns. Then, they proposed a convex decomposition algorithm to partition a concave area into convex subregions by minimizing the sum of the widths of those subregions. Also, they developed a subregion connection algorithm to determine which combination of subregions has the minimum traversal path to cover the entire area.  
 \textcite{Li2011} 
 proposed a method to address the coverage path planning for UAVs in a polygonal area. They showed that, in terms of duration, energy, and path length, turning is an inefficient motion and should be avoided. The authors presented a path with minimal turns for a UAV. Also, they developed a decomposition algorithm to convert a concave area into convex subregions and proposed a subregion connection algorithm to connect adjacent subregions.
 \textcite{di2015energy} proposed an energy-aware path planning algorithm to fully cover a given area while considering other constraints, namely the available energy, the minimum spatial resolution for the pictures, and the maximum camera sampling period. Furthermore, the same authors proposed an energy model for a single UAV based on real measurements, and they used this model to reduce energy consumption in path planning \cite{di2016coverage}. This method assumes back-and-forth motions and determines the UAV's speed to minimize energy consumption, and then the estimated energy consumption is checked for sufficiency to sustain the generated path.
 \textcite{Coombes2018} proposed a new method for planning coverage paths for fixed-wing UAV aerial surveys. They considered windy conditions and proved that flying perpendicular to the wind direction confers a flight time advantage over flying parallel to the wind direction. Additionally, they used dynamic programming to find time-optimal convex decomposition within a polygon.
\textcite{bahnemann2021revisiting} extended the boustrophedon coverage planning by considering different sweep directions in each cell to identify the efficient sweep path. Additionally, they utilized the Equality Generalized Traveling Salesman Problem (E-GTSP) to formulate and find the minimum total path in the adjacency graph.

In addressing the CPP problem, ensuring the complete coverage of the given area is one of the main concerns. This is commonly accomplished by employing cellular decomposition in the area of interest,  where the target area is divided into cells to to make coverage simpler \parencite{choset2001coverage}. In addition to exact cell decomposition, approximate cellular decomposition, or grid-based decomposition, is a common approach used in the literature, where the area of interest is divided into a set of cells, all of which have the same size and shape \parencite{choset2001coverage}. In the grid-based decomposition scheme, some borderline regions outside the area of interest may be included in grid cells, leading to a waste of time and resources which is proportional to the granularity of the grid. 
Despite the loss of accuracy due to discretization, the grid-based decomposition scheme is adopted in many research works. In our opinion, the advantage of grid-based approach includes the ease of managing trajectory conflicts among multiple UAVs, the availability of well-developed frameworks such as mixed-integer programming to analyze the algorithmic performances, and its flexibility to accommodate exact and heuristic enhancements for the path planning task. 
This viewpoint is corroborated by the appreciable volume of studies centered on grid-based decomposition.

 \textcite{Barrientos2011} proposed an integrated tool using a fleet of unmanned aircraft capable of capturing georeferenced images. They considered both regular and irregular grid shapes and utilized the Wavefront algorithm in their approach. Their tool includes several components which are task partitioning and allocation, the CPP algorithm, and robust flight control. 
\textcite{Valente2013} proposed a path planning tool that converts an irregular area into a grid graph and generates near-optimal trajectories for UAVs by minimizing the number of turns in their paths.
\textcite{Nam2016} used grid-based decomposition and wavefront algorithm for a single UAV and considered not only the number of turns but also the route length.
\textcite{Balampanis2017} 
proposed a method to decompose and partition a complex coastal area and generate a list of waypoints for multiple heterogeneous Unmanned Aircraft Systems (UAS) to cover the area. In their method, the search area was partitioned into several subregions, with each sub-region assigned to a UAS for coverage. Furthermore, each subregion was decomposed into a grid of triangular cells, for the designated UAS to follow a moving pattern that connects the borders of the sub-area to the inner regions.
\textcite{Bouzid2017} converted a quadrotor optimal coverage planning problem in areas with obstacles into a Traveling Salesman Problem (TSP) that minimizes the overall energy consumption, and used Genetic Algorithms (GA) to solve it. 
\textcite{Cabreira2019} proposed the Energy-aware Grid-based Covering Path Planning Algorithm (EG-CPP) to minimize the energy consumption of UAVs for covering irregular-shaped areas. They stated that solely considering turns inadequately approximates the energy consumption of UAV paths. Consequently, they improved the cost function proposed in \textcite{Valente2013}, which previously only accounted for the number of turns. This enhancement resulted in a $17\%$ energy saving in real flight tests.
\textcite{Cho2021} considered heterogeneous UAVs in polygon-shaped areas in maritime search and rescue. They proposed a grid-based area decomposition method to convert an area into a graph, and formulated a MIP model on the graph to find the coverage path with minimal completion time. Moreover, they introduced a randomized search heuristic (RSH) algorithm to shorten the computation time for large-scale instances while preserving a small optimality gap. 
In their subsequent work \textcite{Cho2022}, the authors compared the use of hexagonal cells and square cells in a MIP framework, and concluded that the former is more effective for generating paths to minimize the coverage time. The exponential growth in computational complexity with the size of the search area was observed in their experiments. 
In this paper, we also use a MIP model as comparison baseline, but develop an exact algorithm instead of a heuristic one for our main contribution. 
\textcite{Ai2021} used reinforcement learning in coverage path planning for a vessel agent, in which wind and water flow data were collected to make a probability map of the true location of the target in the grid. The authors showed that their approach outperforms some of the previous approaches in three simulated scenarios. \textcite{Song2022} presented a method for UAV path planning, focusing on visiting important points within the search area instead of exhaustive full coverage. They decomposed the area into a grid and assign a score to each cell by using a classifier, generating a new map called heatmap. They tried to enhance the efficiency of rescue operations by incorporating this heatmap into path planning because of the limited endurance of UAVs and the constrained rescue time following a disaster. They proposed three path planning algorithms, conducting comparative analyses among them and an existing algorithm. One of the proposed algorithms demonstrated superior performance; however, none of the algorithms could ensure optimality. \textcite{ahmed2023energy} developed two methods, a greedy algorithm and Simulated Annealing (SA), to address the CPP problem for multiple UAVs. They formulated the problem as a MIP model to minimize energy consumption. They showed for small-scale cases CPLEX had better performance, but in large-scale problems, SA outperformed others to minimize the overall energy consumption.

The contributions of our study are summarized as follows. First, for the multi-UAV CPP problem we present a mathematical formula to calculate the lower bound of the coverage time in windy conditions.  
Second, we propose a constructive algorithm that can rapidly generate near-optimal concurrent coverage paths for large rectangular grids, and prove the solution's near optimality via an exhaustive analysis of data scenarios that may be encountered in the algorithm's execution. 
Third, we conduct through numerical experiments of varying scales to demonstrate the computational advantage and broad usability of the proposed algorithm in addressing complex search and rescue problems utilizing multiple UAVs.

\begin{figure}
    \centering
    \includegraphics[width=\textwidth]{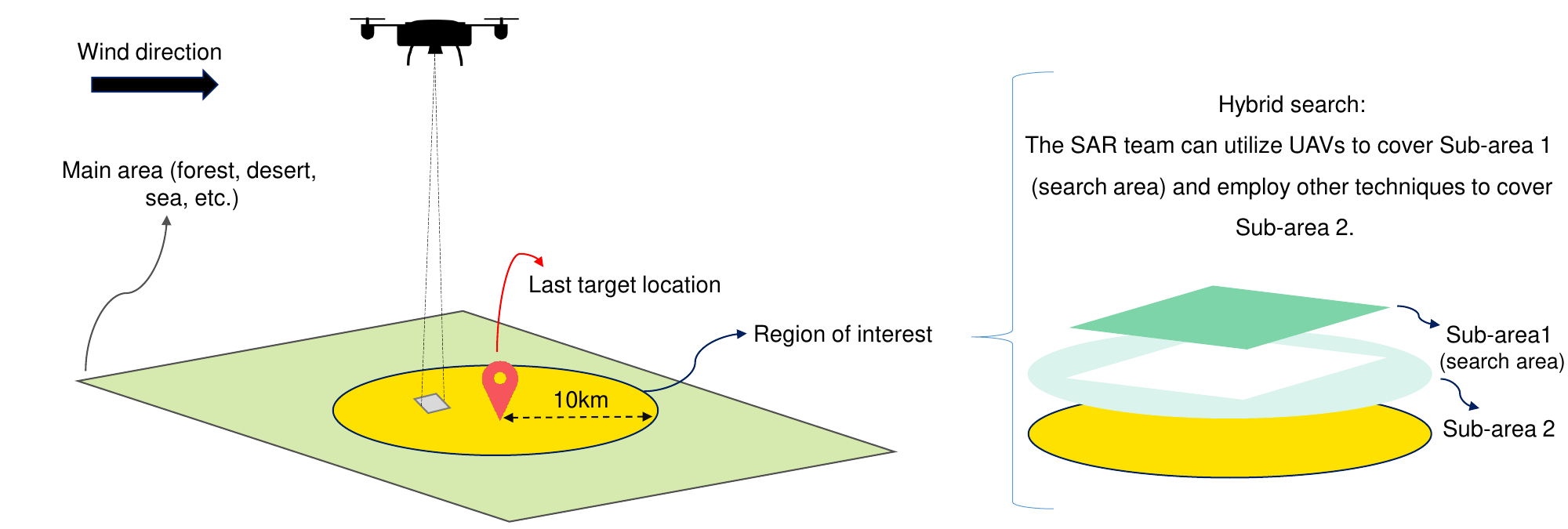}

    \caption{Suppose a SAR team receives a report of a missing mid-age hiker whose last contact was about two hours ago.
 Assuming an average human walking speed of 5 km/h, the radius of the region of interest would be approximately 10 km.
}
    \label{fig:hybrid-search}
\end{figure}

\section{Problem Formulation} \label{sec:problem description}
Consider a typical scenario of searching for a missing person in rural, mountainous and otherwise scantly inhabited areas: The SAR team received a call for help with locating a mid-age hiker whose last approximate location was recorded by a cellular tower two hours ago. By assuming an average walking speed of 5 km/h in this terrain, the team postulated a circular area of 10 km radius that must be searched as quickly as possible. To expedite the search, mixed resources would be utilized, including a fleet of UAVs to cover the central (presumably hard-to-reach) region and canine and on-foot squads to sweep in from the peripheral areas simultaneously.  The division of tasks is demonstrated in Fig. \ref{fig:hybrid-search}. 
In this context, we are concerned with planning the flight paths for the UAVs in a uniform wind field, with the objective of covering (i.e., having a clear aerial image of) all locations in the central rectangular area in minimal time. The UAVs are assumed identical and hence fly at the same airspeed. Collision avoidance among these UAVs must be explicitly considered.

\begin{figure}[h]
    \centering
    \subfigure(a){\includegraphics[width=0.38\textwidth]
    {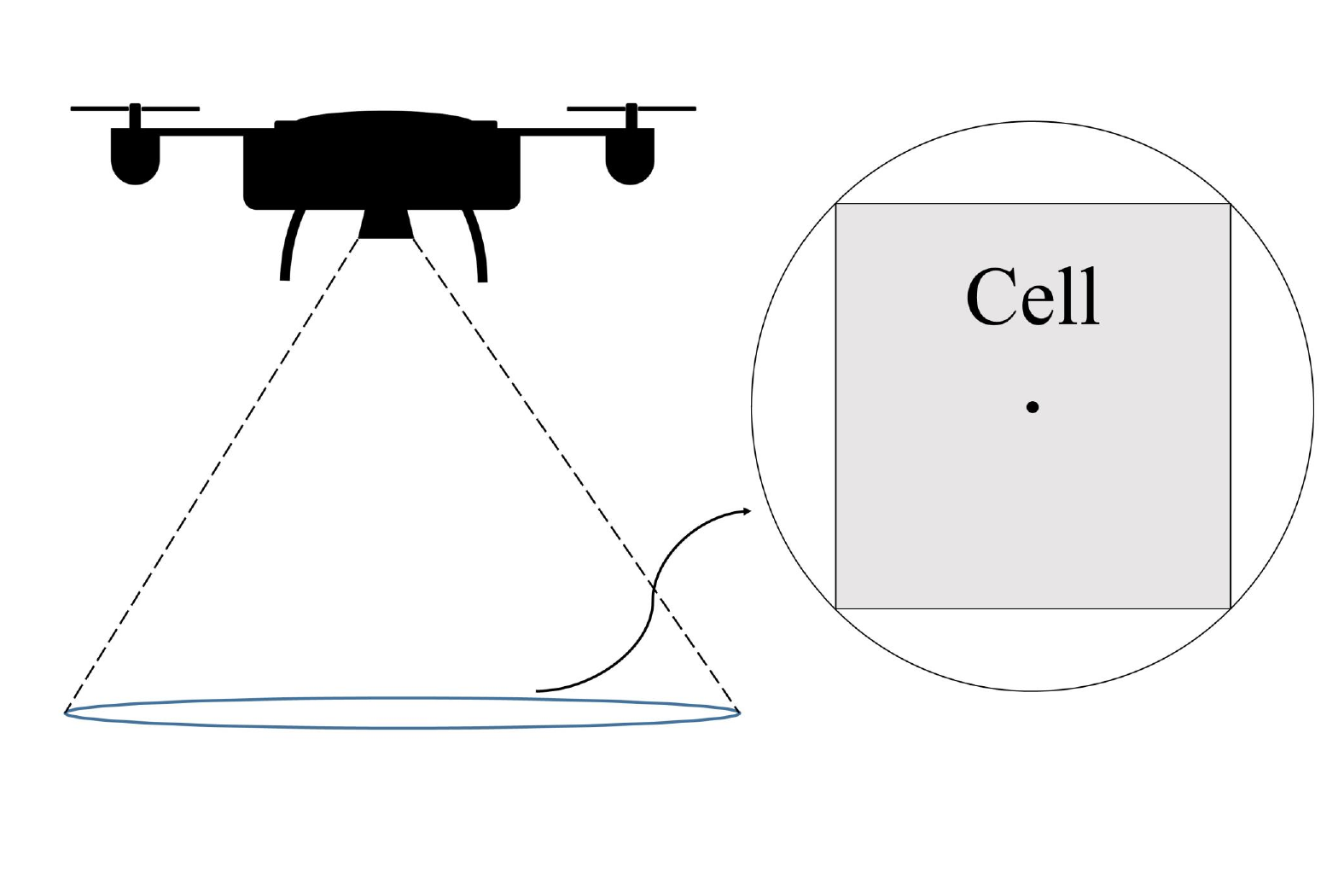}} 
    \hspace{0.4 cm}
    \subfigure(b){\includegraphics[width=0.53\textwidth]{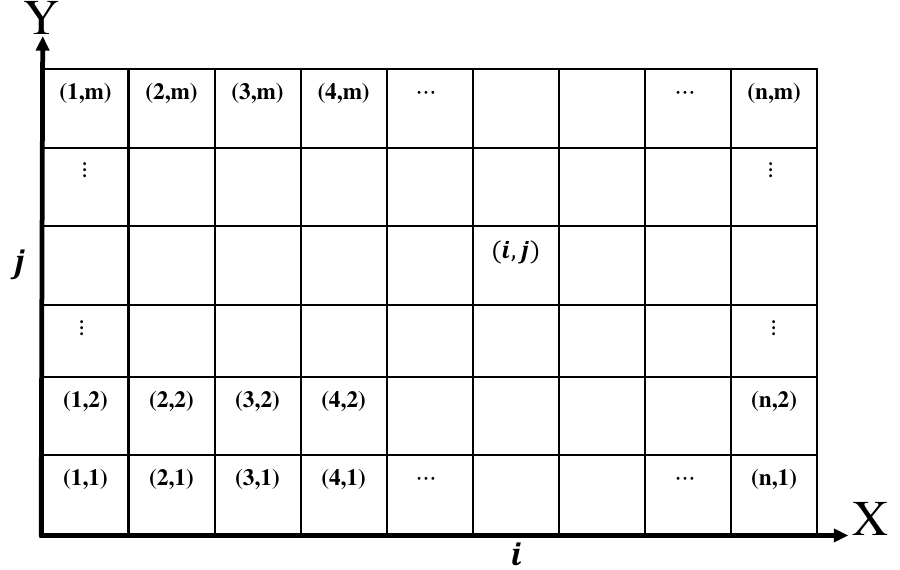}} 
    \caption{(a) The camera's field of view, aligned with a cell in the grid; (b) Cell indexing convention for an ($n \times m$) rectangular search area.}
    \label{fig:footprint&grid}
\end{figure}

\subsection{Grid-based decomposition}\label{gridphase}
To tackle this problem, we first decompose the search area into equal-sized square tiles (or cells) so that each tile fits in the image scope of the UAV's downward facing camera. The actual side length of the tile, together with the flight altitude of the UAV, is determined by factors such as the field-of-view (FoV) angle of the camera, the camera's resolution and the required representation resolution (centimeters of the ground surface per pixel), as demonstrated in Fig. \ref{fig:footprint&grid}a. Pertinent to our modeling framework, each cell is denoted by its central coordinates; for example, cell $(i, j)$ signifies that its central coordinates are located at $(i, j)$, as depicted in Fig. \ref{fig:footprint&grid}b. 

In addition, the following assumptions about the search area is made:
\begin{enumerate}
  \item[] \textbf{Assumption 1:} The $n \times m$ search area represents a rectangular region consisting of $n$ cells along its length (X-axis) and $m$ cells along its width (Y-axis), as depicted in Fig. \ref{fig:footprint&grid}b, and $m$ is greater than or equal to the number of UAVs, denoted by $q$.
  \label{assumption1}
  
  \item[]  \label{assumption2} \textbf{Assumption 2:} 
  The wind direction is parallel to the X-axis of the search area, flowing from west to east, as shown in Fig. \ref{fig:typesMove}. 
  It is noteworthy that the search area's orientation can be arbitrary. Consequently, we can incorporate this assumption into our model without loss of generality.
\end{enumerate}

In the designated search area, two connectivity types are considered. First, the Von Neumann neighborhood enables a UAV to visit adjacent cells labeled as $1, 3, 5$, and $7$ (as shown in Fig. \ref{fig:typesMove}). Second, the Moore neighborhood allows the UAV to visit all eight neighboring cells. Throughout the subsequent sections, our analysis and proposed solution are based on the utilization of the Von Neumann neighborhood. In Sect. \ref{sec:Moore connectivity proof}, we demonstrate the validity of our proposed solution even when employing the Moore neighborhood as the selected connectivity type.
\begin{figure}[h]
    \centering    \includegraphics[width=0.8\textwidth]{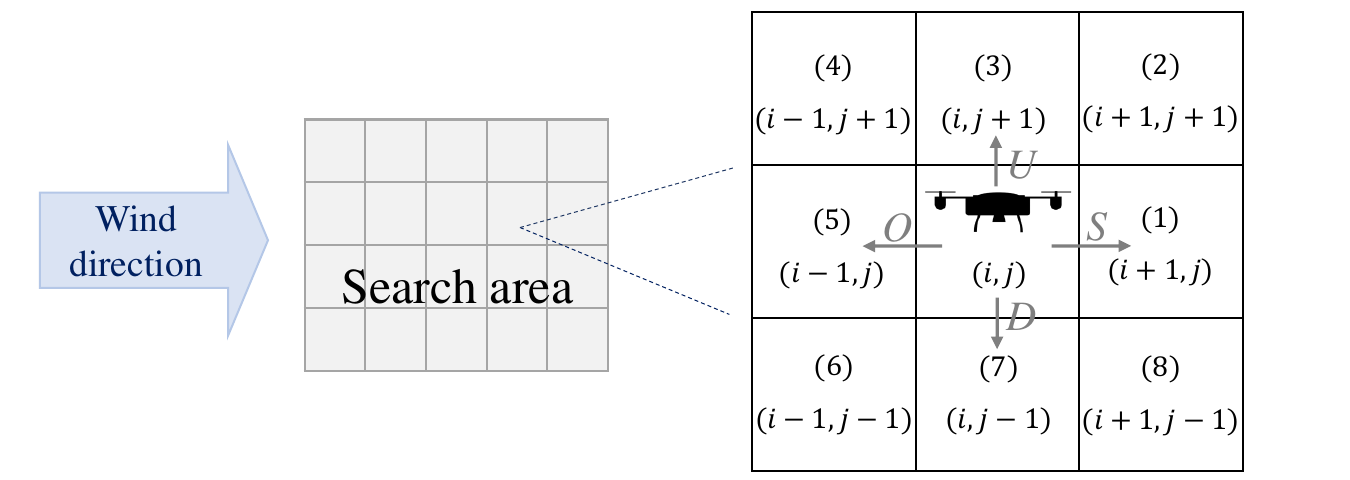}
    \caption{Illustration of cell labeling for neighbors under Assumption 2.}
    \label{fig:typesMove}
\end{figure}

Before continuing to the mathematical formulation, let us define some terms used frequently in the rest of the paper. The term ``$S$ move" means a UAV goes from one cell to its neighboring cell in the wind's direction (cell labeled 1 in Fig. \ref{fig:typesMove}). Conversely, ``$O$ move" refers to a UAV moving from a cell to its neighbor opposite to the wind (cell labeled 5 in Fig. \ref{fig:typesMove}). ``$U$ move" denotes a UAV moving upward to its upper neighbor (cell labeled 3 in Fig. \ref{fig:typesMove}), and ``$D$ move" represents the same action but in the downward direction (cell labeled 7 in Fig. \ref{fig:typesMove}). Additionally, the term ``$P$ move" signifies UAV transitions to neighboring cells perpendicular to the wind's direction (cells labeled 3 or 7 in Fig. \ref{fig:typesMove}); in other words, it's either a $U$ move or a $D$ move. The term \emph{Mission Time} denotes the duration required for a UAV to visit its assigned cells for coverage. Also, the term \emph{Operation Time} represents the duration required for a fleet of UAVs to fully cover a search area. Table \ref{tab:Consistent notation} defines the parameters that are used consistently throughout the paper. 

\begin{table}[h]
\caption{Parameters that define a problem instance} \label{tab:Consistent notation} 
\begin{tabular}{@{}ll@{}}
\toprule
Notations & \\
\midrule
$n$ & Number of cells along the length (X-axis) of the search area\ \\
$m$ & Number of cells along the width (Y-axis) of the search area\ \\
$q$ & Number of UAVs\\
$T_s$ &  Time required for a UAV to perform an $S$ move\\
$T_p$ & Time required for a UAV to perform a $P$ move\\
$T_o$ & Time required for a UAV to perform an $O$ move\\
\botrule
\end{tabular}

\end{table}

\subsection{Mathematical formulation - a baseline}\label{sec:Mathematical mode}

We first formulate the problem as a mixed-integer programming (MIP) model utilizing the notations provided in Table \ref{tab:notation}. This model helps us to characterize the data, decisions and their relations more rigorously, and will serve as a baseline for validating computational results.  \\

\begin{table}[h]
\caption{MIP model notations}\label{tab:notation}
\begin{tabular*}{\textwidth}{@{}ll@{}}
\toprule
Set and parameters & \\
\midrule
 $I$ & Set of X-coordinates for cells\\
    $J$ & Set of Y-coordinates for cells\\
    $K$ & Set of UAVs\\
    $S$ & Set of indices $\{1,2,\ldots,n\times m\}$ for movement steps\\
    \midrule
    Variables & For $i \in I, j \in J, k \in K,\  \text{and}\  s \in S$ where applicable \\
    \midrule
     $x_{kijs}$ & $1$ if UAV $k$ visits cell ($i,j$) in step $s$, otherwise $0$ \\
    $y_{ks}$ &  1 if UAV $k$ is located outside the area in step $s$, otherwise 0\\
    $p_{ks}$ &  $1$ if UAV $k$ moves perpendicular to the wind direction\\
    &from step $s-1$ to $s$, otherwise $0$\\
    $r_{ks}$ &  $1$ if UAV $k$ moves in the wind direction from step $s-1$ to $s$, otherwise $0$\\
    $l_{ks}$ & $1$ if UAV $k$ moves against the wind direction\\
    &from step $s-1$ to $s$, otherwise $0$\\
    $t_\text{opr}$ & Operation time, total time to cover the whole area\\
    \botrule
\end{tabular*}

\end{table}

\begin{align}
    &Min &&t_\text{opr} &&
    \label{eq:mip_obj}\\
    &s.t. &&\sum_{i \in I} \sum_{j \in J} x_{kijs} + y_{ks}=1 \ \ \ \scriptstyle \forall k \in K ,\  \forall s \in S
     \label{eq:mip_c1}\\
    & && \sum_{k \in K} \sum_{s \in S} x_{kijs} =1 \ \ \ \scriptstyle \forall i \in I,\  \forall j \in J
    \label{eq:mip_c2}\\
     & && x_{kijs}-(x_{k(i-1)j(s-1)}+x_{k(i+1)j(s-1)}+x_{ki(j-1)(s-1)}+x_{ki(j+1)(s-1)}) \leq 0 \nonumber\\
     & &&\scriptstyle \forall i \in I,\  \forall j \in J,\  \forall k \in K,\  \forall s \in S \setminus \{1\}
      \label{eq:mip_c3}\\
     & &&x_{kijs}-(y_{k(s+1)}+x_{k(i-1)j(s+1)}+x_{k(i+1)j(s+1)}+x_{ki(j-1)(s+1)}+ \nonumber\\
     & &&x_{ki(j+1)(s+1)}) \leq 0 \ \ \ \scriptstyle \forall i \in I,\  \forall j \in J,\  \forall k \in K,\  \forall s \in S \setminus \{|S|\} 
     \label{eq:mip_c4}\\
    & && \sum_{j \in J} x_{kij(s-1)}+\sum_{j \in J} x_{k(i+1)js}-1 \leq r_{ks} \ \ \ \scriptstyle \forall i \in I \setminus \{|I|\} ,\  \forall s \in S\setminus \{1\},\  \forall k \in K
    \label{eq:mip_c5}\\
    & && \sum_{j \in J} x_{kij(s-1)}+\sum_{j \in J} x_{k(i-1)js}-1 \leq l_{ks} \ \ \ \scriptstyle \forall i \in I \setminus \{1\} ,\  \forall s \in S \setminus \{1\},\  \forall k \in K
    \label{eq:mip_c6}\\
    & && \sum_{i \in I} x_{kij(s-1)}+\sum_{i \in I} x_{ki(j+1)s}-1 \leq p_{ks}
     \ \ \ \scriptstyle \forall j \in J\setminus \{|J|\} ,\  \forall s \in S\setminus \{1\},\  \forall k \in K
     \label{eq:mip_c7}\\
     & &&   \sum_{i \in I} x_{kij(s-1)}+\sum_{i \in I} x_{ki(j-1)s}-1 \leq p_{ks}
     \ \ \ \scriptstyle \forall j \in J\setminus \{1\} ,\  \forall s \in S\setminus \{1\},\  \forall k \in K
     \label{eq:mip_c8}\\
    & && p_{ks}+r_{ks}+l_{ks}+y_{ks}=1
     \ \   \scriptstyle \forall k \in K ,\  \forall s \in S\setminus \{1\}
     \label{eq:mip_c12}\\
    & && t_\text{opr} \geq \  T_p\ .\sum_{s \in S \setminus \{1\}} p_{ks} + \  T_s\ . \sum_{s \in S \setminus \{1\}} r_{ks}+  \  T_o\ . \sum_{s \in S \setminus \{1\}} l_{ks}
    \ \ \ \scriptstyle \forall k \in K  
    \label{eq:mip_c14}
\end{align}
The objective (\ref{eq:mip_obj}) minimizes the operation time, and Constraint (\ref{eq:mip_c1}) ensures that each UAV is either positioned within a cell or outside of the search area.  Constraint (\ref{eq:mip_c2}) guarantees the singular visitation of each cell. Constraint (\ref{eq:mip_c3}) states that a UAV can enter a cell only from its neighboring cells, and Constraint (\ref{eq:mip_c4}) indicates that a UAV can only move to adjacent cells or leave the search area. Constraints (\ref{eq:mip_c5}) and (\ref{eq:mip_c6}) indicate moves aligned with and against the wind direction, respectively; also, Constraints (\ref{eq:mip_c7}) as well as (\ref{eq:mip_c8}) specify moves perpendicular to the wind direction. According to Constraint (\ref{eq:mip_c12}), a UAV within the search area cannot simultaneously execute two or more types of moves.
Finally, Constraint (\ref{eq:mip_c14}) states that the operation time must be greater than or equal to the maximum of UAVs' mission time, which is calculated on the right-hand side of the inequality.

\section{Main Method}\label{sec: solution approach}

In general, MIP problems are intractable and the convergence of branch-and-bound type of algorithms takes exponentially longer as the instance size grows. The above MIP model is no exception - commercial solvers running the above MIP model can only solve small-scale instances (see Section \ref{sec:small_experiments}).

We describe an ingenious approach to significantly expedite the search for the optimal solution. While our algorithm cannot guarantee to always converge to the optimal solution, we can, however, accurately quantify gap between the found solution and the optimal solution. In fact, the optimality gap resulted from our algorithm will be shown to always take one of two values, zero and $T_p$. Therefore, as the search grid grows in size (and hence the optimal value which represents the total time it takes to cover all cells), the relative gap will diminish to zero.

Notations pertinent to the following discussion are defined in Table \ref{tab:prop_notations}. 
Let us examine an equivalent formulation for the above Multiple UAV Coverage Path Planning (MUCPP) problem in the $n\times m$ search area. This step is essential to facilitate the subsequent development and implementation of our proposed solution.
\begin{table}[h]
 \caption{MUCPP model and propositions notations}
    \label{tab:prop_notations}
\begin{tabular*}{\textwidth}{@{}ll@{}}
\toprule
Symbol & Meaning\\
\midrule
$\bar{C}$ & Set of cells in the $n \times m$ search area, \{$(1,1),(1,2),...,(n,m)$\} \\ 
 $N_c$ & Set of neighboring cells for cell $c$ where $c\in \bar{C}$, $|N_c|\leq 4$\\ 
 $p_k$ & Sequence of $k$ ordered cells, $(c_{(1)},c_{(2)},...,c_{(k)})$, in which\\
 & $c_{(i)}\in \bar{C}\  \text{for}\  i=1,2,3,...,k$, $c_{(i+1)}\in N_{c_{(i)}}\ \text{for}\ i=1,2,...,k-1$, called path of length $k$\\
 $p_k^{\prime}$ & Set of cells within path $p_k$ \\
 $P_k$ & Set of all paths of length $k$\\
 $P$ & Set of all paths, $\cup_k P_k$\\
 $t_{c,c^{\prime}}$ & Time taken to move from cell $c$\ to cell $c^{\prime}$, where $c\in \bar{C}$ and $c^{\prime} \in N_c$ ($t_{c,c^{\prime}} \in \{T_s,T_o,T_p\}$)\\
 $t_{p_k}$ & Time taken to cover cells in path $p_k$ by a UAV, equal to $\sum_{i=1}^{k-1} t_{c_{(i)},c_{(i+1)}}$ for $p_k \in P_k$\\
 $d_i$ & Number of cells assigned to UAV $i$, where $i \in \{1,2,\dots,q\}$\\
 $\vec{v}_a$&  2-D vector representing the airspeed of the UAV\\
 $\vec{v}_w$ & 2-D vector representing the wind speed\\
 $\vec{v}_g$ & 2-D vector representing the ground speed of the UAV\\
\botrule
\end{tabular*}

\end{table}

 \begin{align}
 &\text{MUCPP}: && \nonumber\\
    & \min_{p_{d_1},p_{d_2},...,p_{d_q}} && \max_{i \in \{1,2,...,q\}}\  t_{p_{d_i}}  &&&\label{eq:obj-MUC}\\
   &s.t. &&  \  p_{d_i}^{\prime} \cap p_{d_j}^{\prime} = \varnothing   &&& \forall i,j \in \{1,2,\dots,q\}\  \text{and}\  i\neq j 
   \label{eq:c1-MUC}\\
   & && \  \cup_{i=1}^q p_{d_i}^{\prime} = \bar{C}  
   \label{eq:c2-MUC}\\
   & &&  \  p_{d_i} \in P &&& \forall i \in \{1,2,\dots,q\}
   \label{eq:c3-MUC}
\end{align}   

The objective (\ref{eq:obj-MUC}) is to minimize the maximum time taken by UAVs to cover cells in their assigned path. Constraint (\ref{eq:c1-MUC}) ensures that paths do not intersect and that all cells are visited at most once. Constraint (\ref{eq:c2-MUC}) requires that the solution must cover all cells, and Constraint (\ref{eq:c3-MUC}) ensures the connectivity of UAVs' paths.
To derive a lower bound for the MUCPP problem, denoted as ``LB", we relax Constraints (\ref{eq:c1-MUC}) and (\ref{eq:c2-MUC}) to create a relaxed version of the problem called R-MUCPP (Relaxed Multiple UAVs Coverage Path Planning).
\begin{center}
 \begin{align}
    &\text{R-MUCPP}: && &&& \nonumber\\
    &\min_{p_{d_1},p_{d_2},...,p_{d_q}}  && \max_{i\in \{1,2,...,q\}}\  t_{p_{d_i}} &&& 
    \label{eq:obj-RMUC}\\
   &s.t. && \sum_{i=1}^{q} d_i = nm &&& d_i \in \mathbb{N}
   \label{eq:c1-RMUC}\\
    &  && p_{d_i} \in P &&& \forall i \in \{1,2,\dots,q\}
    \label{eq:c2-RMUC}
\end{align}   
\end{center}

In the R-MUCPP problem, the objective (\ref{eq:obj-MUC}) and Constraint (\ref{eq:c3-MUC}) are inherited from the original MUCPP formulation. Constraint (\ref{eq:c1-RMUC}) requires that the total number of cells assigned to UAVs must be equal to the number of cells within the search area. If a solution satisfies Constraints (\ref{eq:c1-MUC}) and (\ref{eq:c2-MUC}), it is evident that it will also satisfy Constraint  (\ref{eq:c1-RMUC}). In the next step, Section \ref{sec:LBproof} presents two propositions that form the basis for obtaining the solution to the R-MUCPP problem. The objective function of R-MUCPP will serve as LB for MUCPP.
\subsection{Lower bound derivation} \label{sec:LBproof}
Let $\Bar{D}$ represent the distance between the centers of two adjacent cells, and assume $\|\vec{v}_a\| = v_a$,  $\|\vec{v}_w\| = v_w$, and $v_a > v_w \geq 0$. Then, the values of $T_s$, $T_p$, and $T_o$ can be calculated using Eq. (\ref{eq:defTspo}).
\begin{center}
\begin{align}
 &T_s=\frac{\Bar{D}}{v_a+v_w}
 &&T_p=\frac{\bar{D}}{\sqrt{v_a^2-v_w^2}}
 &&T_o=\frac{\bar{D}}{v_a-v_w} \label{eq:defTspo}
\end{align}
\end{center}
Furthermore, we have $T_s\leq T_p \leq T_o$ because of $v_a+v_w \geq \sqrt{v_a^2-v_w^2} \geq v_a-v_w$. 
To introduce and prove the LB for the MUCPP problem, we need to establish some preliminary results.

\begin{lemma}\label{lemma:0}
Let $A$, $B$, $K$, $H$, and $N$ be non-negative real numbers satisfying $A \le B \le K$, $A + K \ge 2B$, and $H \ge N$. Then, the optimal value for $(x_1,x_2,x_3)$ in the following linear programming problem is $(x_1^*,x_2^*,x_3^*)=(N,H-N,0)$.
\begin{center}
    \begin{align}
        & \min && Ax_1+Bx_2+Kx_3  \label{eq:lem1obj} \\
        & s.t. && x_1-x_3\leq N  \label{eq:lem1c1}\\
        & && x_1+x_2+x_3=H \label{eq:lem1c2} \\
        &  && x_i\ge 0 && i \in \{1,2,3\}
    \end{align}
\end{center}

\begin{proof}
\[
\begin{split}
& \min_{x_1, x_2, x_3 \ge 0} A x_1 + B x_2 + K x_3 \\
= & \min_{x_1, x_3 \ge 0} A x_1 + B (H - x_1 - x_3) + K x_3 \\
= & \min_{x_1, x_3 \ge 0} (A-B) x_1 + (K-B) x_3 + BH \\
\ge & \min_{x_3 \ge 0} (A-B) (N + x_3) + (K-B) x_3 + BH \\
= & \min_{x_3 \ge 0} (A+K-2B) x_3 + (A-B)N + BH \\
= & AN + (H-N)B \\
\end{split}
\]
where the first equality is by applying (\ref{eq:lem1c2}), the second equality is by collecting terms, the inequality is by (\ref{eq:lem1c1}) and the assumption that $(A-B) \le 0$, and the last step is because of $A + K - 2B \ge 0$ and the nonnegativity of $x_3$. The lower bound of the objective value is attained when $x_3 = 0$ and the third line holds at equality, i.e., when $x_1$ is equal to  $x_3 + N$.  The solution $(x_1, x_2, x_3) = (N, H-N, 0)$ is feasible and makes the objective function achieves a valid lower bound, therefore, it is optimal.  
\end{proof}
\end{lemma}
\newtheorem{prop}{Proposition}
\begin{prop}\label{prop:1}
The time taken to fly a path of length $d$ is at least $T$, defined as\\

$
    T= 
\begin{cases}
    (d-1) T_s& \  \text{if } d \leq n\\
    (n-1) T_s+(d-n) T_p              & \  \text{if } d > n
\end{cases}
$
\end{prop}

\begin{proof}
We know that 
\begin{center}
    \begin{align*}
       & p_{d}=(c_{(1)},c_{(2)},...,c_{(d)})\\
       & t_{p_{d}}=\sum_{i=1}^{d-1} t_{c_{(i)},c_{(i+1)}}\\
       & t_{c_{(i)},c_{(i+1)}}\in \{T_s,T_p,T_o\}
    \end{align*}
\end{center}
Consider the first state ($d\leq n$), since $t_{c_{(i)},c_{(i+1)}}$ can take three values, the minimum of the $t_{p_{d}}$ would be attained if we choose $T_s$ for $t_{c_{(i)},c_{(i+1)}}$, for all $i \in \{1,\dots,d-1\}$. Hence, in this state, the mission time, $t_{p_{d}}$, cannot be less than $(d-1) T_s$.

In the second case, for a clearer explanation, the formula to calculate $t_{p_d}$ can be written as $t_{p_d}=T_s x_1+T_p x_2+T_o x_3$, where $x_1$, $x_2$, and $x_3$ are defined as the number of $S$ moves, $P$ moves, and $O$ moves, respectively, with the constraint $x_1+x_2+x_3=d-1$.

By the definitions given in Eq. (\ref{eq:defTspo}), we have 
\[
T_s + T_o = \frac{2v_a\bar{D}}{v_a^2-v_w^2} \ge \frac{2\bar{D}}{\sqrt{v_a^2-v_w^2}} = 2 T_p
\]
Furthermore, we have $T_s \le T_p \le T_o$ by assumption. Therefore, Lemma \ref{lemma:0} can be applied; for this purpose, as it is mentioned before, consider $x_1$ as the number of $S$ moves, $x_2$ as the number of $P$ moves, $x_3$ as the number of $O$ moves, $T_s$ as $A$ or $x_1$ coefficient, $T_p$ as $B$ or $x_2$ coefficient, and $T_o$ as $K$ or $x_3$ coefficient.
In addition, it is supposed that the length of the area has $n$ cells, which means the summation of the $S$ moves and $O$ moves should always be less than $n-1$ to guarantee that the UAV will stay in the search area, it can be stated $x_1-x_3 \leq n-1$. It is clear that $x_1+x_2+x_3=d-1$ due to the length of the path, and additionally, it is considered $d-1$ as $H$ and $n-1$ as $N$ in Lemma \ref{lemma:0}.
Based on the Lemma \ref{lemma:0} solution, the optimal values are $x_1^*=n-1, x_2^*=d-n, x_3^*=0$; therefore, the minimum possible time for a path of length $d$ is $(n-1) T_s+(d-n) T_p$.
\end{proof}

\begin{lemma}\label{lemma:1}
 For any positive integer $N$, if we express it as the sum of $d$ non-negative integers, then the largest summand in the summation cannot be smaller than $\lceil \frac{N}{d} \rceil$.     

\begin{proof}
    Assume for contradiction that the largest component is equal to $Q := \lceil \frac{N}{d} \rceil - k, (k \geq 1)$. There is a $\alpha \in [0,1) $ such that $\lceil \frac{N}{d} \rceil = \frac{N}{d}+\alpha$, so $Q$ can be written as $Q=\frac{N}{d}+\alpha-k$. All of the $d$ summands are less than or equal to $Q$, so their sum cannot exceed $dQ = N +d(\alpha-k)$, a quantity strictly less than $N$. This is a contradiction.  
\end{proof}
\end{lemma}

Now, we can state a key proposition as follows. 

\begin{prop}\label{prop:3}
  Let $n,m,q$ be positive integers satisfying $q\le m$, then the time that is required to cover all cells of the $n\times m$ search area by $q$ UAVs cannot be less than $(n-1)T_s+(\lceil\frac{nm}{q}\rceil-n) T_p$.
\end{prop}
\begin{proof}
It is clear that the time required to fully cover the $n\times m$ search area by $q$ UAVs, i.e., operation time, must be greater than or equal to the maximum mission time of the UAVs. Based on Lemma \ref{lemma:1}, we establish that the maximum number of cells assigned to UAVs must be no less than $\lceil \frac{nm}{q} \rceil$. Consequently, in accordance with Proposition \ref{prop:1}, the maximum mission time of the UAVs cannot be less than $(n-1)T_s+(\lceil\frac{nm}{q}\rceil-n) T_p$.
\end{proof}
According to Proposition \ref{prop:3}, the optimal solution of the R-MUCPP problem, acting as the lower bound of the MUCPP problem (LB), can be derived using the following formula: $(n-1)T_s + (\lceil\frac{nm}{q}\rceil - n) T_p$ (referred to as the ``LB formula" throughout the rest of this paper).

In Sect. \ref{sec:propositional algorithm}, we will employ the LB formula as a key insight in the development of the solution algorithm for the MUCPP problem.

 \subsection{Near-optimal path planning algorithm (NOPP)}\label{sec:propositional algorithm}

We develop a constructive algorithm that consistently generates feasible solutions with an objective value in the set $\{\text{LB},\text{LB}+T_p\}$, where LB represents the lower bound established for the MUCPP problem (explained in Sect. \ref{sec:LBproof}). 
Our proposed algorithm comprises four distinct phases, and we employ the notations detailed in Table \ref{tab:algorithm notation} to explain its functionality within the $n \times m$ search area. 
In this section, we describe the four phases of the algorithm and leave the detailed analysis and proof of its near optimality to the Appendix.

\begin{table}[]
\caption{Notations used in the NOPP algorithm} \label{tab:algorithm notation} 
\begin{tabular*}{\textwidth}{@{}ll@{}}
\toprule
Symbol & Meaning \\
\midrule
$F$ & Set of uncovered cells\\
$C$ & Arbitrary cell, in the search area, Its coordinates (coordinates of the center)\\
&are indicated by $(c_x, c_y)$\\
$\mathit{NP}$ & Number of available (remaining) $P$ moves for a UAV\\
$\mathit{NU}$ & Number of $U$ moves made by a UAV\\
$\mathit{ND}$ & Number of $D$ moves made by a UAV\\
$A$ & Number of cells assigned to a UAV to cover\\
$\mathit{SP}$ & Cell from which a UAV starts its path, referred as \emph{Starting point}\\
$\Bar{O}$ & Binary variable denotes the parity of the initial value of $\mathit{NP}$: True for odd, False for even \\
$c_x$ & X-coordinate of cell $C$\\
$c_y$ & Y-coordinate of cell $C$\\
$path$ & Sequence of cells' coordination traversed and covered by a UAV\\
$path^{\ast}$ & Sequence of moves executed by a UAV, elements must be ``U", ``D", or ``S" such that:\\
&\ \ \ $\bullet$ ``U" means $U$ move\\
&\ \ \ $\bullet$ ``D" means $D$ move\\
&\ \ \ $\bullet$ ``S" means $S$ move\\
$\cup^{\ast}$ & Operator between two sequences with the following definition: let $S_1$ and $S_2$ be two \\
&sequences such that $S_1=(a_1,a_2,\dots,a_z)$ and $S_2=(b_1,b_2,\dots,b_w)$, for $z,w \in \mathbb{N}$; then, \\
& $S_1 \cup^{\ast} S_2=(a_1,a_2,\dots,a_z,b_1,b_2,\dots,b_w)$\\
\botrule
\end{tabular*}
\footnotetext{Note: For further clarification regarding $path$ and $path^\ast$, let's consider a scenario in which a UAV starts from the bottom-left corner of Fig. \ref{fig:footprint&grid}b at cell $(1,1)$ and covers the subsequent four cells above it. In this case, $path$ is defined as $((1,1),(1,2),(1,3),(1,4),(1,5))$, while $path^\ast$ corresponds to the UAV's four $U$ moves, represented as $(\text{``U"},\text{``U"},\text{``U"},\text{``U"})$.}
\end{table}

\subsubsection{Initial settings}\label{sec:initial setting}

Initial values of the variables involved in the algorithm are set as follows: $A$ and $\mathit{NP}$ are assigned $\lceil \frac{nm}{q} \rceil$ and $\lceil \frac{nm}{q} \rceil -n$ , respectively (derived from Proposition \ref{prop:3}). Additionally, the variable $\bar{O}$ is assigned either True or False based on the parity of $\mathit{NP}$. Furthermore, the starting point $\mathit{SP}$ for the $i^{th}$ UAV, where $\forall i \in \{1,2,\dots,q\}$, is at the coordinate $(1,m-q+i)$. Lastly, the set $F$ encompasses all cells within the $n \times m$ search area, except for those that have already been covered.

\subsubsection{Phase one}\label{sec:phase1}

In this phase, the UAV's path (a sequence of cell coordinates covered by the UAV) starts from the starting point determined in the Sect. \ref{sec:initial setting}. In each iteration, the algorithm evaluates the feasibility of the $D$, $S$, and $U$ moves in order, and selects the first feasible move. In addition to being unvisited, a cell must satisfy the following conditions to be a feasible destination for the next move: (1) The UAV cannot move to a cell whose Y-coordinate is greater than that of the starting point, referred as \emph{Y-coordinate condition} in the rest of the paper; (2) the total number of $P$ moves (the sum of $D$ and $U$ moves) made by a UAV must be less than or equal to  $\lceil \frac{nm}{q} \rceil - n$, based on Proposition \ref{prop:3}; and (3) if $\bar{O}$ is True, the maximum count of $U$ moves is constrained to be one less than the count of $D$ moves.
The pseudo code for Phase 1 is given in Algorithm \ref{alg:phaseOne}.

\begin{algorithm}
\caption{Phase one of NOPP}\label{alg:phaseOne}
\begin{algorithmic}[1]
\algnotext{EndIf}
\algnotext{EndFor}
\State \textbf{function} \ phaseOne($\mathit{SP},\mathit{NP},F,A,\bar{O}$)
\State Initialize $path$\ and \ $path^\ast$ to be empty sequence, $\mathit{NU} \gets 0, \ \mathit{ND} \gets 0$, $C \gets \mathit{SP}$
\State $s_{2} \gets \text{the second element (Y-coordinate) of } \mathit{SP}$, $path \gets path \cup^{\ast} (C)$, and $F \gets F \backslash \{C\}$ 
\For{ $1\  \text{through}\  A$}
\If{$(c_x,c_y-1) \in F\  \text{and}\ \mathit{NP} >0$}
\State $C \gets (c_x,c_y-1) $
\State $path \gets path \cup^{\ast} (C)$, $F \gets F \backslash \{C\}$, $path^\ast \gets path^\ast \cup^{\ast} (\text{``D"})$
\State $\mathit{NP} \gets \mathit{NP} - 1,\  $ $\mathit{ND} \gets \mathit{ND}+1$
\ElsIf{$(c_x+1,c_y) \in F\ $}
\State $C \gets (c_x+1,c_y)$
\State $path \gets path \cup^{\ast} (C)$, $F \gets F \backslash \{C\}$, $path^\ast \gets path^\ast \cup^\ast (\text{``S"})$ 
\ElsIf{$(c_x,c_y+1) \in F\  \text{and}\  \mathit{NP}>0\  \text{and}\  \mathit{NU}<\mathit{ND} \ \text{and } c_y+1\le s_{2}$ }

\If{$\Bar{O}$}
\If{$\mathit{NU}<\mathit{ND}-1$}
\State $C \gets (c_x,c_y+1)$
\State $path \gets path \cup^{\ast} (C)$, $F \gets F \backslash \{C\}$, $path^\ast \gets path^\ast \cup^\ast (\text{``U"})$
\State  $\mathit{NP} \gets \mathit{NP} - 1,\ \mathit{NU} \gets \mathit{NU}+1$
\EndIf
\Else
\State $C \gets (c_x,c_y+1)$
\State $path \gets path \cup^{\ast} (C)$, $F \gets F \backslash \{C\}$, $path^\ast \gets path^\ast \cup^\ast (\text{``U"})$ 
\State $\mathit{NP} \gets \mathit{NP} - 1,\ \mathit{NU} \gets \mathit{NU}+1$
\EndIf
\EndIf
\EndFor
\State \textbf{return} $path,path^\ast,F,\mathit{NP}$
\end{algorithmic}
\end{algorithm}
\subsubsection{Phase two}
\label{seq:pahseTwo}
This phase can be executed only when the $path^*$ (generated by Phase 1) concludes with a sequence of $U$ moves, denoted as $U^*$. Derived from the structure of Phase 1, a sequence of $S$ moves consistently precedes $U^*$, termed as $S^*$. We define the \emph{H-value} as $|S^*| - 1$ when the move prior to $S^*$ is a $D$ move, and as $|S^*|$ otherwise. Furthermore, if the $path^*$ does not conclude with a sequence of $U$ moves, the H-value assumes a value of zero, as depicted in Fig. \ref{figervalley}b. If the H-value is an odd number, Phase 2 is implemented, during which the algorithm removes $U^*$ from the $path^*$ and updates $path$, $path^{\ast}$, $F$, and $\mathit{NP}$—all of which were obtained from Phase 1. Algorithm \ref{alg:phaseTwo} shows the pseudocode for Phase 2.
\begin{algorithm}[h]
\caption{Phase two of NOPP}\label{alg:phaseTwo}
\begin{algorithmic}[1]
\algnotext{EndIf}
\algnotext{EndFor}
\State \textbf{function} \ phaseTwo$(path,path^*,F,\mathit{NP})$
\State \text{Initialize}\  $pathCopy \gets path, i \gets 0$
\State $i \gets |path^\ast|-|U^*|$
\State $path \gets$ the first $i+1 $ elements in $path$, $path^\ast \gets$ the first $i$ elements in $path^\ast$
\State $F\gets F \cup \{\text{elements of $pathCopy$ excluding the first $i+1$ elements}\}$
\State $\mathit{NP} \gets \mathit{NP}+|U^{\ast}|$
\State \textbf{return}\  $path,path^\ast,F,\mathit{NP}$
\end{algorithmic}
\end{algorithm}
\subsubsection{Phase three} \label{sec:phase3}

This phase is implemented when the count of remaining $P$ moves ($\mathit{NP}$) is greater than zero. During this phase, the algorithm traverses the $path^\ast$ to detect two specific subsequences, (``S", ``S") and (``U", ``S"), referred to as \emph{BI patterns}. It then assesses the feasibility of inserting a single $U$ move before the second move of the BI patterns. The insertion of a $U$ move is considered feasible when a UAV can move upward to visit a cell that has not been covered previously and satisfies the Y-coordinate condition (as described in Sect \ref{sec:phase1}). If this insertion is feasible, the algorithm inserts one $U$ move before the second move of the BI patterns as well as one $D$ move after that, and updates $path^\ast$ accordingly. In other words, (``S", ``S") is replaced with (``S", ``U", ``S", ``D"), (``U", ``S") is replaced with (``U", ``U", ``S", ``D") in $path^\ast$. This operation is iterated $\mathit{NP}/2$ times by the algorithm.

 It should be mentioned that as per the Phase 3 design, $\mathit{NP}$ must be an even number as the input of this phase. Hence, if $\mathit{NP}$ is an odd number after Phases 1 and 2, it is automatically incremented by one and then passed to Phase 3. This adjustment might lead the operation time to reach $\text{LB}+T_p$. The pseudocode for Phase 3 is presented by Algorithm \ref{alg:phase3}.
\begin{figure}[h]
    \centering
    \subfigure(a){\includegraphics[width=0.3\textwidth]{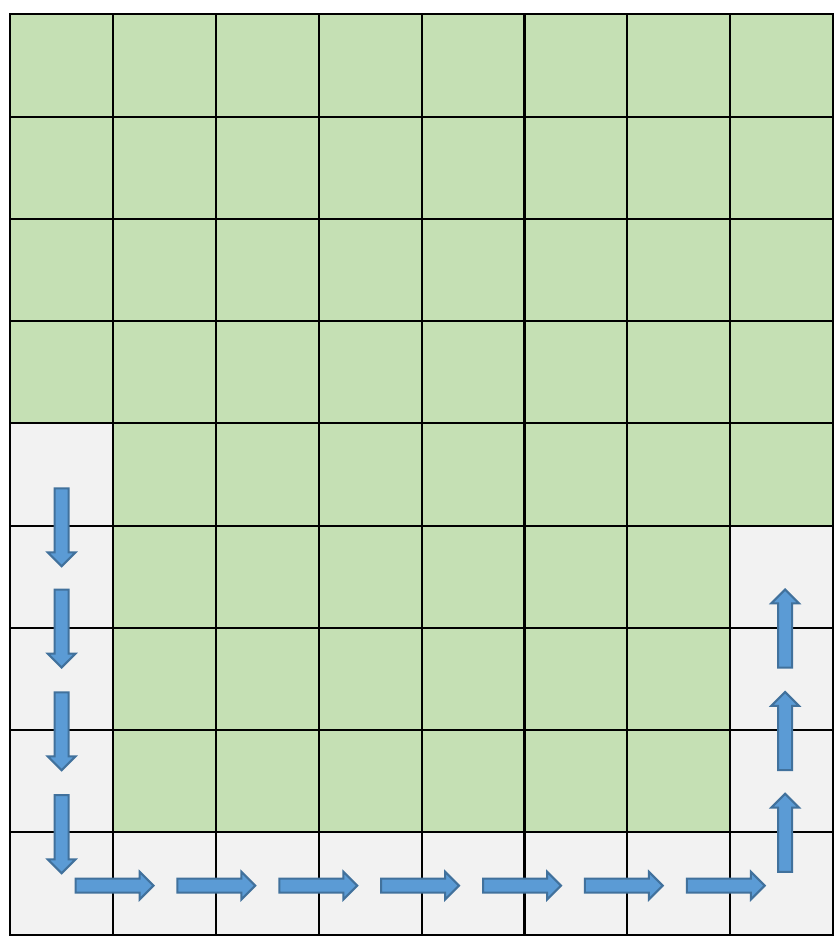}} 
    \hspace{1 cm}
    \subfigure(b){\includegraphics[width=0.3\textwidth]{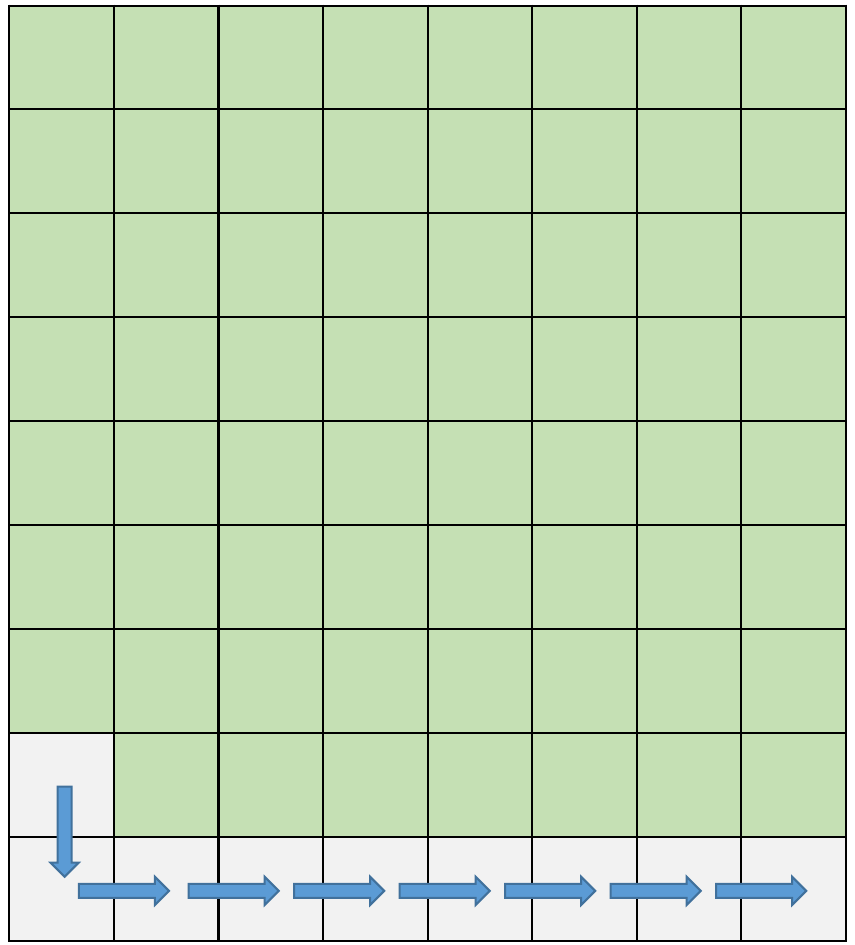}}
    \caption{Example of calculating the H-value for the first UAV after phase 1 in two cases:  (a) Due to $|S^*|=7$ and a $D$ move before $S^*$, the H-value is 6 ($n=8,\ m=9,\ q=5$); (b) Since there is no $U^*$, the H-value is 0 ($n=8,\ m=9,\ q=8$).}
    \label{figervalley}
\end{figure}
\begin{algorithm}
\caption{Phase three of NOPP}\label{alg:phase3}
\begin{algorithmic}[1]
\algnotext{EndIf}
\algnotext{EndFor}
\State \textbf{function} \ phaseThree$(\mathit{SP},path,path^*,F,\mathit{NP})$
\State \text{Initialize}\ $ I \gets 1,\  C \gets \mathit{SP} $, $s_2 \gets \text{the second element (Y-coordinate) of } \mathit{SP}$ 
\For{$1\  \text{through}\  \frac{\mathit{NP}}{2}$}
\For{$i \in \{1,\dots,|path^\ast|\}$}
\If {$i > I $ \text{and} $ i^\text{th} \ \text{element of } path^* \text{ is ``S''}$ \text{and} $ (i-1)^\text{th} \ \text{element of } path^* \in \text{\{``U'',``S"\}}$}
\State $C \gets \text{the $i^{th}$ element of } path$

\If{$(c_x,c_y+1) \in F$ and $c_y+1 \le s_2$}
\State \text{Update $path^{\ast}$ by placing a ``U'' and a ``D" before and after the $ i^\text{th} \ \text{element of } path^*$, respectively}
\State \text{Update $path$ by placing $(c_x,c_y+1)$ followed by $(c_x+1,c_y+1)$ after $C$ in $path$}
\State $F \gets F \backslash \{(c_x,c_y+1),(c_x+1,c_y+1)\}, \mathit{NP} \gets \mathit{NP}-2, I \gets i$
\State \textbf{break}
\Else
\State \textbf{continue}
\EndIf
\EndIf
\EndFor
\EndFor
\State $\textbf{return}\  path, path^*, F, \mathit{NP}$

\end{algorithmic}
\end{algorithm}
\subsubsection{Phase four}
Due to the Phase 3 design, the algorithm may generate a UAV path where some cells with an X-coordinate equal to $n$ remain uncovered, leaving them for the next UAV to cover. This is not important for other UAVs except for the last one. Since all cells must be covered, the last UAV cannot leave any cells uncovered because there are no other UAVs available to cover them. Therefore, Phase 4 is exclusive to the last UAV and operates only when $F$ is not empty. In other words, the purpose of Phase 4 is to ensure that uncovered cells are covered by the last UAV. Algorithm \ref{alg:phase4} presents the pseudocode for Phase 4.

\subsubsection{Generating paths for UAVs}

The process flow of the NOPP algorithm is illustrated in  Fig. \ref{fig:flowchart}. The different phases described above are stringed together to form the complete NOPP algorithm in the following way. 

The NOPP algorithm consists of running the above described phases in order for each UAV. 
For a given UAV $i$, Phase 1 is executed after the initial settings, then the H-value  (defined in Sect. \ref{seq:pahseTwo}) is calculated. If the H-value is odd, Phase 2 is invoked; otherwise, Phase 2 is omitted for this UAV. Subsequently, the number of remaining $P$ moves (i.e., the $\mathit{NP}$ value which is updated in Phases 1 and 2) is checked. If it exceeds zero and is an even number, Phase 3 is directly initiated, and if it is an odd number, then it is incremented to the next even number before Phase 3 is initiated. Finally, if $NP$ is zero, Phase 3 is omitted for this UAV. 
For the last UAV (i.e., $i = q$), after Phase 3 (if initiated at all), Phase 4 will be executed only if the set $F$ is not empty. This phase ensures that all uncovered cells at this point will be covered by the last UAV.

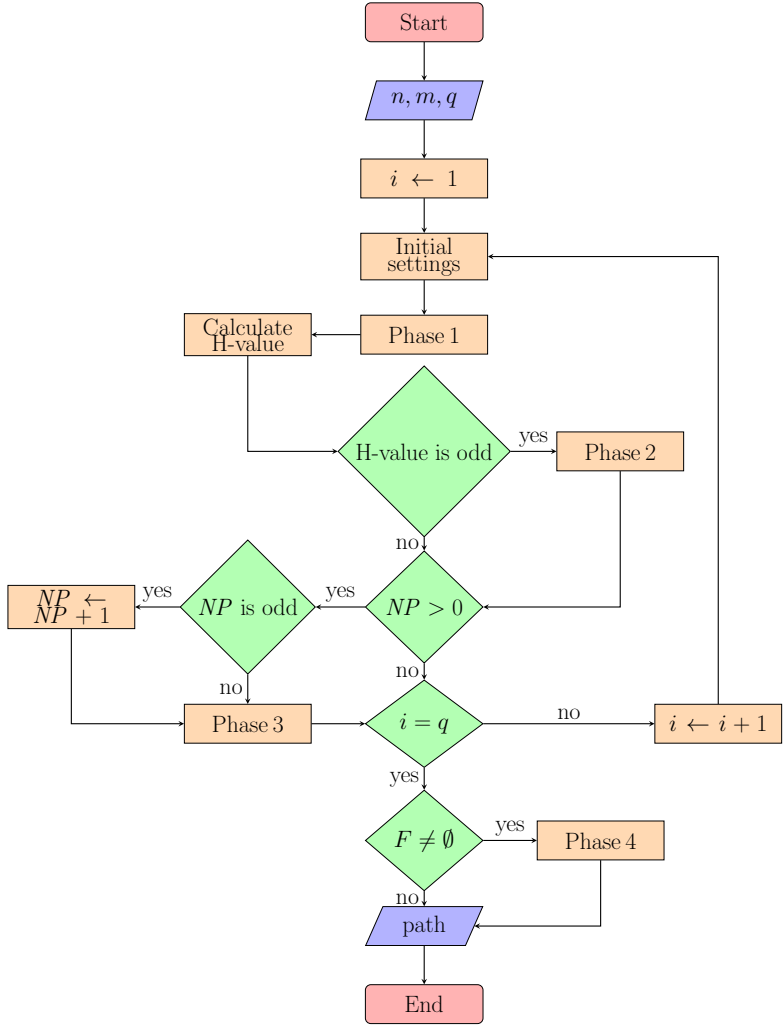
\begin{figure}[h]
\begin{adjustbox}{max height=0.55\textheight,center} 
\begin{tikzpicture}[node distance=2cm]\

\node (start) [startstop] {\Large Start};
\node (in1) [io, below of=start,yshift=-0.2] {\textbf{\Large $ n,m,q$}};
\node (pro) [process, below of=in1,yshift=-0.2] {\Large $i\leftarrow1$};
\node (pro0) [process, below of=pro,yshift=-0.2] {\Large Initial settings};
\node (pro1) [process, below of=pro0,yshift=-0.2] {\Large Phase 1};
\node (pro1-1)[process, left of=pro1,xshift=-2.5 cm] {\Large Calculate H-value};
\node (dec1) [decision, below of=pro1, yshift=-1cm] {\Large H-value is odd};
\node (pro2) [process, right of=dec1, xshift=3 cm] {\Large Phase 2};
\node (dec2) [decision, below of=dec1,yshift=-2cm] {\Large $\mathit{NP} > 0 $};
\node(dec2-1)[decision, left of=dec2,xshift=-2.5cm] {\Large $\mathit{NP}\ \text{is odd} $};
\node(pro2-1)[process, left of=dec2-1, xshift=-2.5 cm] {\Large $\mathit{NP} \gets \mathit{NP}+1$};
\node (pro3) [process, below of=dec2-1, yshift=-1 cm] {\Large Phase 3};
\node (dec3) [decision, below of=dec2,yshift=-1cm] {\Large $i=q$};
\node (pro4) [process, right of=dec3, xshift=5.5 cm] {\Large $i \leftarrow i+1$};
\node (dec4) [decision, below of=dec3,yshift=-1cm] {\Large $F\neq  \emptyset$};
\node (prop5) [process, right of=dec4, xshift=2.5 cm]{\Large Phase 4};
\node (in2) [io, below of=dec4,yshift=-0.2cm] {\Large path};
\node (end) [startstop,below of=in2,yshift=-0.2] {\Large End};

\draw [arrow] (start) -- (in1);
\draw [arrow] (in1) -- (pro);
\draw [arrow] (pro) -- (pro0);
\draw [arrow] (pro0) -- (pro1);
\draw [arrow] (pro1) -- (pro1-1);
\draw [arrow] (pro1-1) |- (dec1);
\draw [arrow] (dec1) -- node[anchor=east] {\Large no} (dec2);
\draw [arrow] (dec1) -- node[anchor=south] {\Large yes} (pro2);
\draw [arrow] (pro2) |- (dec2);
\draw [arrow] (dec2) -- node[anchor=east] {\Large no} (dec3);
\draw [arrow] (dec2) -- node[anchor=south] {\Large yes} (dec2-1);
\draw [arrow] (dec2-1) -- node[anchor=south] {\Large yes} (pro2-1);
\draw [arrow] (dec2-1) -- node[anchor=east] {\Large no} (pro3);
\draw [arrow] (pro2-1) |- (pro3);
\draw [arrow] (pro3) -- (dec3);
\draw [arrow] (dec4) -- node[anchor=south] {\Large yes} (prop5);
\draw [arrow] (dec3) -- node[anchor=east] {\Large yes}(dec4);
\draw [arrow] (dec3) --node[anchor=south] {\Large no} (pro4);
\draw [arrow] (pro4) |- (pro0);
\draw[arrow] (dec4)-- node[anchor=east] {\Large no}(in2);
\draw [arrow] (prop5) |- (in2);
\draw [arrow] (in2) -- (end);

\end{tikzpicture}
\end{adjustbox}
\vspace{0.1cm}
    \caption{The flowchart of the NOPP algorithm.}
    \label{fig:flowchart}
\end{figure}

\begin{algorithm}
\caption{Phase four of NOPP}\label{alg:phase4}
\begin{algorithmic}[1]
\algnotext{EndIf}
\algnotext{EndFor}
\State \textbf{function}\ phaseFour($path,F$)
\For{1 through $|F|$}
\State $C \gets$ the last element of $path$
\If{$(c_x,c_y+1) \in F$}
\State $path \gets path \cup^\ast ((c_x,c_y+1))$
\EndIf
\EndFor
\State $\textbf{return}\  path$

\end{algorithmic}
\end{algorithm}

Before proceeding to the next section, it's important to mention that the proposed algorithm can handle cases where $q \le n$, which are common in real-world applications. 

Note that the algorithm is developed under the implicit assumption that $q \le n$, which represents the majority of large-scale real-world cases. In cases with $q > n$ (meaning that we have more UAVs than the number of cells along the wind direction), we can decompose the main problem into multiple subproblems with $q \le n$, and apply the proposed algorithm to each subproblem with little loss of efficiency.

\section{Extension to Moore Neighborhood Connectivity} \label{sec:Moore connectivity proof}

In this section, we demonstrate the effectiveness of our proposed solution when we employ the Moore neighborhood as the selected connectivity type. To do so, we only need to prove that Proposition \ref{prop:1} holds true in the Moore neighborhood connectivity. In the beginning, we introduce several new definitions. Considering Fig. \ref{fig:typesMove}, we define an ``$F$ move" as a UAV's movement from cell $(i,j)$ to adjacent cells numbered 2 and 8. Likewise, a ``$B$ move" refers to the UAV's transition from cell $(i, j)$ to adjacent cells numbered 4 and 6. Furthermore, $T_f$ and $T_b$ denote the time necessary for executing $F$ and 
$B$ moves, respectively. 
Let $\bar{D}$ be the distance from cell $(i, j)$ to its neighboring cell numbered 1 (or equivalently 3, 5, and 7); likewise, $\sqrt{2}\bar{D}$ be the distance to neighboring cell numbered 2 (or equivalently 4, 6, and 8). Additionally, we employ the notations $\Vec{v}_a, \Vec{v}_g, \text{and } \Vec{v}_w$ (as defined in Table \ref{tab:prop_notations}). Then, we define $\|\vec{v}_a\| = v_a$, $\|\vec{v}_g\| = v_g$, and $\|\vec{v}_w\| = v_w$ and assume that $0<v_w<v_a$. \\
\begin{prop}
In the Moore neighborhood connectivity where $T_s < T_f < T_p < T_b < T_o$, the minimum time to cover $d$ cells (T), as stated in proposition \ref{prop:1}, remains to be\\

$
    T= 
\begin{cases}
    (d-1) T_s& \  \text{if } d \leq n\\
    (n-1) T_s+(d-n) T_p              & \  \text{if } d > n
\end{cases}
$

\end{prop}
\begin{proof}
    When $d \leq n$, it is evident that a path with $d-1$ sequential $S$ moves achieves the shortest time, equal to $(d-1)T_s$.
    In case where $d>n$, based on Proposition \ref{prop:1}, we know that a path with $n-1$ $S$ moves and $d-n$ $P$ moves, called \emph{VN path}, has the minimum time in Von Neumann connectivity, which is $(n-1)T_s+(d-n)T_p$. Suppose that we can create a faster path than the VN path in the Moore neighborhood connectivity. To do so, we need to replace a $P$ move in the VN path with either an $S$ move or an $F$ move because both $T_s$ and $T_f$ are less than $T_p$, thus making the path faster.

    In the first case, we examine replacing a $P$ move in the VN path with an $S$ move. This action creates a new path that is faster than the VN path, and we refer to it as the \emph{Moore path}. The Moore path consists of $n$ $S$ moves and $d-n-1$ $P$ moves, resulting in a total path time of $(n)T_s + (d-n-1)T_p$, which is less than $(n-1)T_s + (d-n)T_p$. Note 
    that the search area is composed of $n$ cells along its length, and employing $n$ $S$ moves within the Moore path results in the UAV exiting the search area. To ensure the UAV remains within the designated search area, it is imperative to substitute a $P$ move in the Moore path with either one $B$ or $O$ move. Given $T_b<T_o$, we update the Moore path by replacing a $P$ move with a $B$ move. After this modification, the Moore path consists of $n$ $S$ moves, one $B$ move, and $d-n-2$ $P$ moves, which takes $(n)T_s + T_b + (d-n-2)T_p$ to cover $d$ cells. To show that the Moore path is faster than the VN path, we have to demonstrate that $(n)T_s+T_b+(d-n-2)T_p < (n-1)T_s+(d-n)T_p$, or equivalently, $T_s+T_b-2T_p<0$. We know that $T_s$ and $T_p$ are attainable through Eq. (\ref{eq:defTspo}), so we only need to determine $T_b$ by calculating $v_g$ in a $B$ move. To find $v_g$ in a $B$ move, we define $\alpha$ as the angle between $\Vec{v}_a$ and $\Vec{v}_g$, and $\beta$ as the angle between $\Vec{v}_a$ and $- \Vec{v}_w$ (as depicted in Fig. \ref{fig:MooreProof}). In the first step, we calculate $\alpha$ using the law of sines and then attain $\beta$. So, we have:\\ 
        \[
\begin{split}
& \frac{sin(\frac{3\pi}{4})}{v_a}=\frac{sin(\alpha)}{v_w}\  \Rightarrow \ \alpha=arcsin(\frac{\sqrt{2}v_w}{2v_a})  \\
& \beta=\pi-\frac{3\pi}{4}-\alpha=\frac{\pi}{4}-arcsin(\frac{\sqrt{2}v_w}{2v_a})
\end{split}
\]\\
Then, $v_g$ in a $B$ move can be calculated by the following formula:\\
        \[
\begin{split}
& v_g=\sqrt{v_a^2+v_w^2-2 v_a v_w cos(\frac{\pi}{4}-arcsin(\frac{\sqrt{2}v_w}{2 v_a}))}
\end{split}
\]\\
Then, $T_b$ is:\\
        \[
\begin{split}
& T_b=\frac{\sqrt{2}\bar{D}}{\sqrt{v_a^2+v_w^2-2v_a v_w cos(\frac{\pi}{4}-arcsin(\frac{\sqrt{2}v_w}{2 v_a}))}}
\end{split}
\]\\
By using  $v_w=\gamma v_a$ such that $\gamma \in (0,1)$, we can write $T_s+T_b-2T_p$ as the following function:

\[
\begin{split}
& f(v_a,\gamma)=T_s+T_b-2T_p \\
& =\frac{\bar{D}}{v_a+\gamma v_a}+\frac{\sqrt{2}\bar{D}}{\sqrt{v_a^2+(\gamma v_a)^2-2\gamma v_a^2 cos(\frac{\pi}{4}-arcsin(\frac{\sqrt{2}\gamma v_a}{2 v_a}))}}-\frac{2\Bar{D}}{\sqrt{v_a^2-\gamma^2 v_a^2}}\\
&= \frac{\Bar{D}}{v_a}(\frac{1}{1+\gamma}+\frac{\sqrt{2}}{\sqrt{1+\gamma^2-2\gamma cos(\frac{\pi}{4}-arcsin(\frac{\sqrt{2}\gamma}{2}))}}-\frac{2}{\sqrt{1-\gamma^2}})
\end{split}
\]\\
For simplicity, we can write $f(v_a,\gamma)=\frac{\bar{D}}{v_a} g(\gamma)$ such that
\[
\begin{split}
    & g(\gamma)=\frac{1}{1+\gamma}+\frac{\sqrt{2}}{\sqrt{1+\gamma^2-2\gamma cos(\frac{\pi}{4}-arcsin(\frac{\sqrt{2}\gamma}{2}))}}-\frac{2}{\sqrt{1-\gamma^2}}
\end{split}
\]
By considering $g$ function, we can see that $g(0)>0$, and also, for $\gamma \in (0,1)$, the first derivative of $g$ function is positive ($g^{\prime}(\gamma)>0$). Therefore, we conclude that for $\gamma \in (0,1)$,  $g(\gamma)$ is an increasing function and always positive. Given $\frac{\bar{D}}{v_a}>0$,it is evident that $f(v_a, \gamma)>0$ for $\gamma \in (0,1)$. Therefore, $T_s+T_b-2T_p>0$ and $T_s+T_b>2T_p$ which leads to a contradiction. It means that the Moore path is slower than the VN path.

In the second case, we consider replacing a $P$ move in the VN path with an $F$ move. Through this action, we create a new path called the \emph{F-Moore path}, which we assume is faster than the VN path. Like the first case, we need to use either a $B$ or $O$ move to keep the UAV inside the search area. Since $T_b<T_o$, we replace a $P$ move in the F-Moore path with a $B$ move and update it. After this update, the F-Moore path includes $n-1$ $S$ moves, one $F$ move, one $B$ move, and $d-n-2$ $P$ moves, which takes $(n-1)T_s+T_f+T_b+(d-n-2)T_p$ to cover $d$ cells. It is clear that $T_f+T_b > T_s+T_b$, and, in the first case, we proved that $T_s+T_b > 2T_p$. Therefore, it is concluded that $T_f+T_b>2T_p$ which result in $(n-1)T_s+T_f+T_b+(d-n-2)T_p>(n-1)T_s+(d-n)T_p$ and leads to contradiction. It means that the VN path is faster than the F-Moore path.

In the next step, it is necessary to discuss using an $O$ move instead of a $B$ move in the Moore path. Suppose we utilize an $O$ move instead of a $B$ move in the Moore path to maintain the UAV within the search area. Therefore, the Moore path would include $n$ $S$ moves, one $O$ move, and $d-n-2$ $P$ moves. We can observe that the Moore path only utilizes $S$, $O$, and $P$ moves. Thus, the Moore path is one of the available paths in the Von Neumann neighborhood connectivity to cover $d$ cells. In Proposition \ref{prop:1}, we have proved that the VN path is the fastest path among all available paths in this type of neighborhood connectivity. Therefore, $(n)T_s+T_o+(d-n-2)T_p \ge (n-1)T_s+(d-n)T_p$, or equivalently, $T_s+T_o \ge 2T_p$. Now, it is straightforward to consider using an $O$ move instead of a $B$ move in the F-Moore path to keep the UAV inside the search area. By this change, the F-Moore path consists of $n-1$ $S$ moves, one $F$ move, one $O$ move, and $(d-n-2)$ $P$ moves. Given $T_f > T_s$, it is evident that $T_f + T_o > T_s + T_o$, so we can write $T_f + T_o > 2T_p$, as we have shown that $T_s + T_o \geq 2T_p$. This implies that $(n-1)T_s + T_f + T_o + (d-n-2)T_p > (n-1)T_s + (d-n)T_p$, and therefore, the VN path is still faster than the F-Moore path.

Finally, since it is not possible to create a feasible path faster than the VN path within the Moore neighborhood connectivity, Proposition \ref{prop:1} remains valid even when the Moore neighborhood connectivity is selected.
    
\end{proof}

\begin{figure}[h]
    \centering    \includegraphics[width=0.5\textwidth]{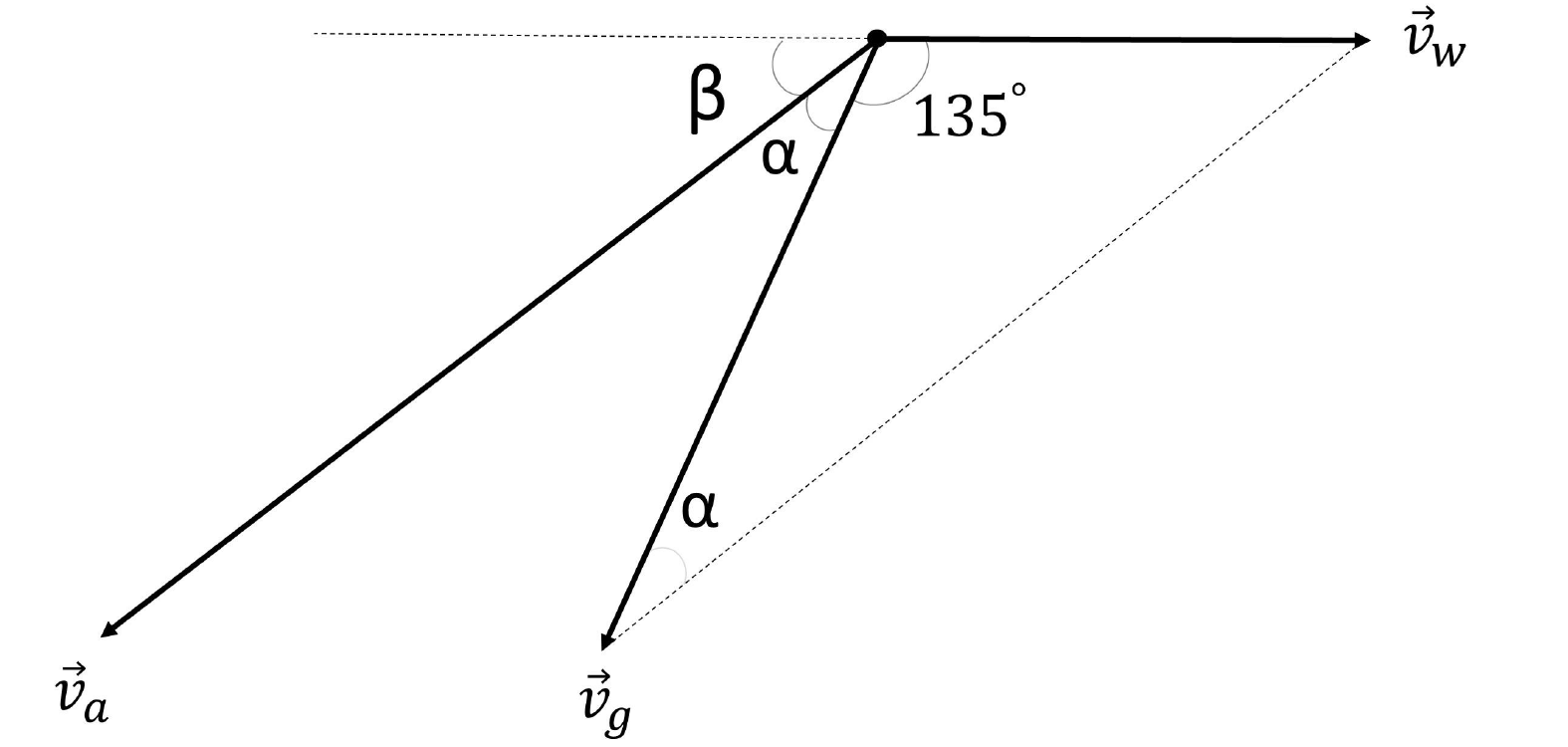}
    \caption{Vectors of the UAV's air and ground speed as well as wind speed for a $B$ move.}
    \label{fig:MooreProof}
\end{figure}

\section{Numerical Experiments} \label{sec:experimentalRes}
The experiments are performed on a PC with a Core i7-7700 processor and 32 GB RAM. The proposed algorithms are implemented in Python 3.7, whereas the MIP models are written in GAMS and solved by CPLEX. 

We first experiment on small test cases to verify the solution obtained from different approaches, and to compare the computational time. 
Then, four medium-sized instances are solved to showcase the algorithm's performance, and two of them are used to demonstrate how each phase of the NOPP algorithm works. Finally, cases of varying numbers of UAVs and different grid sizes are solved to analyze the effect of problem size on computational complexity. 

To determine cell dimensions, recent advancements in AI technology (\cite{martinez2021search}) establish target recognition capabilities at distances of 100 meters during flight at a height of 20 meters. Leveraging cutting-edge high-resolution cameras, utilizing a grid cell size of $100 \times 100$ meters square becomes a viable choice. The experiments are conducted with UAV and wind speeds set at 20 m/s and 5 m/s, respectively; therefore, transition times between adjacent cells are computed at 4 seconds, 5.16 seconds, and 6.66 seconds for $T_s$, $T_p$, and $T_o$. Codes are available at \href{https://github.com/Sina14KD/CPP-SearchRescue}{https://github.com/Sina14KD/CPP-SearchRescue}.
\subsection{Small-sized cases and Comparison with commercial solver} \label{sec:small_experiments}
In this section, we evaluate the objective function value through both the utilization of the commercial solver, CPLEX, and our proposed algorithm in scenarios comprising a total of 10 to 100 cells. Given the complexity of the model, a computational ceiling of 3600 seconds is imposed on the CPLEX, whereby the objective value is established as the minimum achieved within this temporal threshold. Also, due to the significant computational time required by CPLEX, we limit our experiments to simple scenarios that involve only two UAVs. Furthermore, the experiments are executed within an $n\times m$ search area utilizing a fleet of $q$ UAVs, and this configuration is represented as a sequence of three values in the initial column of Table \ref{tab:small-case-result}. Subsequent columns display the lower bound and results, including the objective value, computational time, and optimality gap.
\begin{table}[h]
\caption{Results for small-sized test cases}\label{tab:small-case-result}
\begin{tabular*}{\textwidth}{@{\extracolsep\fill}lcccccccc}
\toprule%
Case$(n,m,q)$&LB& \multicolumn{3}{@{}c@{}}{MIP by CPLEX} & \multicolumn{2}{@{}c@{}}{NOPP} &Optimality& Absolute\\\cmidrule{3-5}\cmidrule{6-7}%
& & $Z$ & $\text{LB}_{\mathit{CPLEX}}$ &Time(s) & $Z_{\mathit{NOPP}}$ & Time(s) & gap($\%$)& gap(s) \\
\midrule
(4,4,2)&32.64&32.64&32.64&8.32 &32.64&0.00&0.00&0.00\\
(4,5,2)&42.96&42.96&42.96&111.90 &42.96&0.00&0.00&0.00\\
(5,4,2)&41.80&41.80&41.80&52.35& 46.96&0.00&12.34&5.16\\
(5,5,2)&57.28&57.28&57.28&1507.46& 62.44&0.00&9.00&5.16\\
(6,5,2)&66.44&66.44&60.00&$>3600$& 66.44&0.00&0.00&0.00\\
(5,6,2)&67.60&67.60&60.06&$>3600$& 67.60&0.00&0.00&0.00\\
(6,6,2)&81.92&82.26&69.16&$>$3600& 81.92&0.00&0.00&0.00\\
(7,7,2)&116.88&123.22&94.24&$>$3600& 122.04&0.00&4.41&5.16\\
(8,8,2)&151.84&161.36&124.00&$>$3600& 151.84&0.00&0.00&0.00\\
(9,9,2)&197.12&403.52&158.00&$>$3600& 202.28&0.00&2.62&5.16\\
(9,10,2)&217.76&-&176.00&$>$3600&217.76&0.00&0.00&0.00\\
(10,9,2)&216.60&-&176.00&$>$3600&216.60&0.00&0.00&0.00\\
(10,10,2)&242.40&-&196.00&$>$3600& 242.40&0.00&0.00&0.00\\

\botrule
\end{tabular*}
\footnotetext{LB is obtained from LB formula in Sect. \ref{sec:LBproof}. $\text{LB}_{\mathit{CPLEX}}$ is the lower bound of the CPLEX, optimality gap = $\frac{|LB-Z_{\mathit{NOPP}}|}{LB}\times 100$, \\and absolute gap = $|LB-Z_{\mathit{NOPP}}|$.}
\end{table}
Table \ref{tab:small-case-result} reveals the substantial inefficiency of the CPLEX algorithm, even when applied to small-sized cases. In contrast, the proposed algorithm rapidly generates solutions belonging to the set $\{\text{LB}, \text{LB}+T_p\}$ across all cases.
\subsection{Medium-sized cases and visual explanation}
In this section, we investigate medium-sized cases, defined as instances comprising a cell count ranging from 100 to 1000. The outcomes for four specific cases are presented in Table \ref{tab: mid case result}, where the initial two cases are used to elucidate the algorithm's functioning. Cases three and four each encompass 1000 cells with different configurations, serving to verify that the algorithm efficiently resolves medium-sized cases within remarkably brief computational intervals. Additionally, as expected, the results demonstrate that all solutions belong to the set $\{\text{LB}, \text{LB}+T_p\}$.

\begin{table}[h]
\caption{Results for Medium-sized test cases}\label{tab: mid case result} 
\begin{tabular*}{\textwidth}{@{\extracolsep\fill}lccccccc}
\toprule%
Case no.&Size of the&Number&LB& \multicolumn{2}{@{}c@{}}{NOPP} &Optimality & Absolute \\\cmidrule{5-6}%
 & $n\times m$ search area & of UAVs & & $Z_{\mathit{NOPP}}$ & Time(s) & gap($\%$)& gap(s)\\
\midrule
1&$11\times 10$&3&174.16&179.32&0.00&2.9&5.16\\
2&$13\times 11$&4&166.68&166.68&0.00&0.00&0.00\\
3&$25\times 40$&2&2547&2552.16&0.25& $\approx 0.00$&5.16\\
4&$50\times 20$&6&779.72&779.72&0.15&0.00&0.00\\
\botrule
\end{tabular*}
\footnotetext{LB is obtained from LB formula in Sect. \ref{sec:LBproof}. Optimality gap = $\frac{|LB-Z_{\mathit{NOPP}}|}{LB}\times 100$, and absolute gap = $|LB-Z_{\mathit{NOPP}}|.$}
\end{table}
For the first case (case no. 1 in Table \ref{tab: mid case result}) depicted in Fig. \ref{fig:case1midsize}, the algorithm initially plans a path for UAV number one. Preceding the initiation of Phase 1, according to initial settings, the starting point is designated as $(1,10-3+1)$, and the total number of assigned cells ($A$) and available $P$ moves ($\mathit{NP}$) are 37 and 26 respectively. In the next step, Phase 1 starts its task and descends from the starting point until reaching $(1,1)$,  where the $D$ move becomes unavailable. Then, it continues by $S$ moves till $(11,1)$, and since the $S$ move is no longer available, it goes up until $(11,8)$. Next, the H-value is calculated, resulting in an odd value of 9. Consequently, the algorithm, during Phase 2, removes $U^*$ from the path, leaving a path comprising 7 sequential $P$ moves and 10 consecutive $S$ moves. In the subsequent step, the algorithm checks the residual count of $P$ moves ($\mathit{NP}$) that is 19 (26 minus 7), and given its odd value, the algorithm automatically applies a one-unit incremental adjustment, setting $\mathit{NP}$ to 20. Then, in accordance with the pattern established during Phase 3, the allocation of these 20 $P$ moves takes place along the remaining path from Phase 2, creating the ultimate path of the first UAV. For the next UAV or UAV number $2$, the starting point is  $(1,10-3+2)$, and the number of assigned cells and available $P$ moves remain consistent (37 and 26, respectively). Subsequently, in Phase 1, the algorithm constructs a path following the pattern of this phase, which concludes at cell $(11, 9)$. Then, The algorithm calculates the H-value, and since it is 4 (an even number), Phase 2 is omitted. In the next step, given that $\mathit{NP}$ has an even value of 12 (26 minus 14), Phase 3 accomplishes its designated task by placing 12 $P$ moves into the path. Regarding the last UAV or UAV number 3, the algorithm first executes Phases 1, 2, and 3 sequentially, and then, as the last UAV has four phases, the algorithm checks the set of uncovered cells ($F$). Finally, since there is no uncovered cell, Phase 4 is not implemented.

Figure \ref{fig:case2midsize} illustrates the algorithm's performance corresponding to the second case (case no. 2 in Table \ref{tab: mid case result}), where, for all four UAVs, phases 1, 2, and 3 operate similarly to the first case. However, unlike the first case, the $\bar{O}$ is true (because the initial value of $\mathit{NP}$ is 23), and therefore, after Phase 1 the number of $U$ moves is one less than $D$ ones. Also, after implementing Phase 3 for the last UAV, set F is not empty and includes cell (13,11). Consequently, during Phase 4, the algorithm employs a $U$ move to cover cell $(13,11)$.

\begin{figure}
\includegraphics[width=1\textwidth]{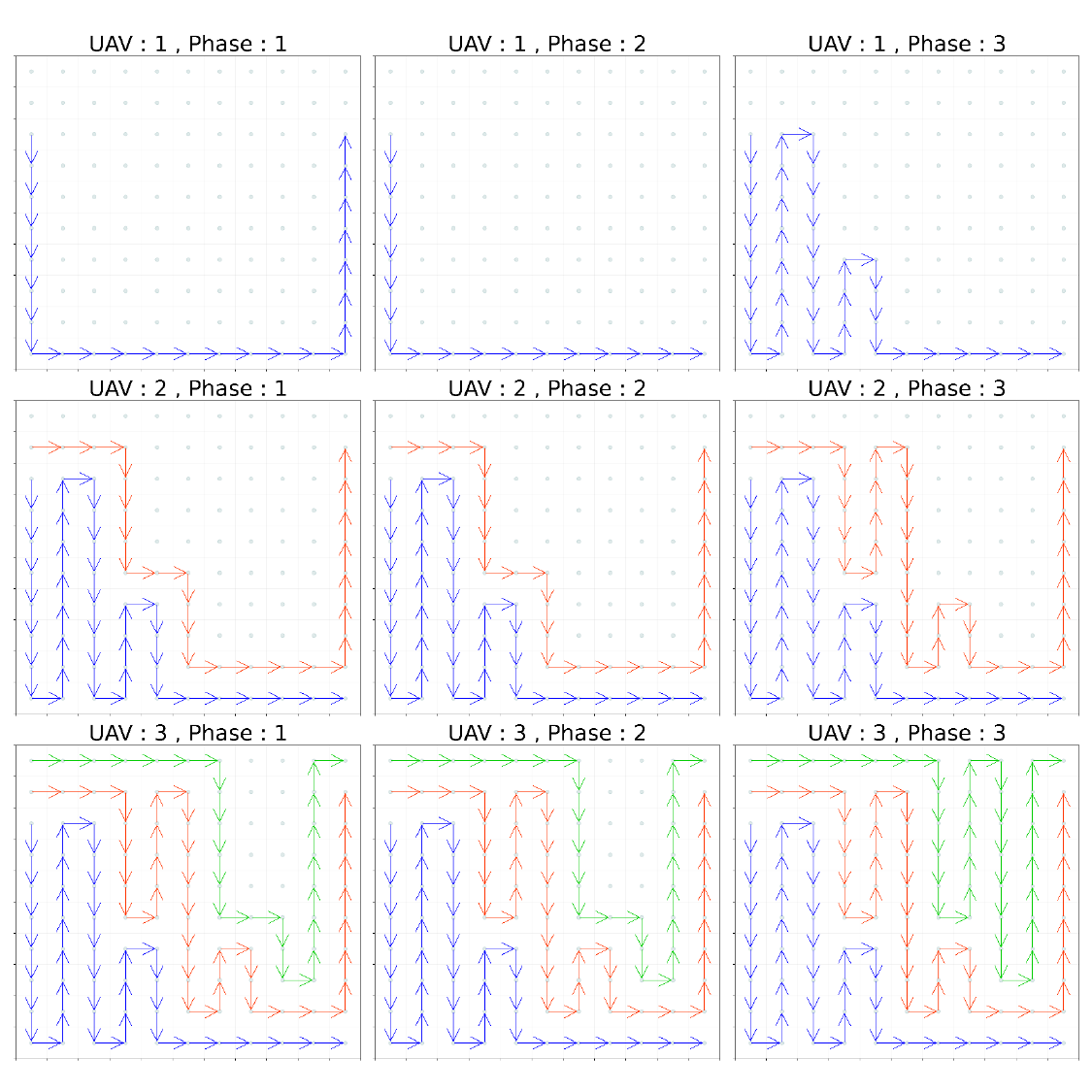}
    \caption{Demonstration of the algorithm's phases for 3 UAVs in $11\times10$ search area (Case no.\hspace{0.2em}1 in Table \ref{tab: mid case result}.)}
    \label{fig:case1midsize}
\end{figure}
\begin{figure}
\includegraphics[width=1\textwidth]{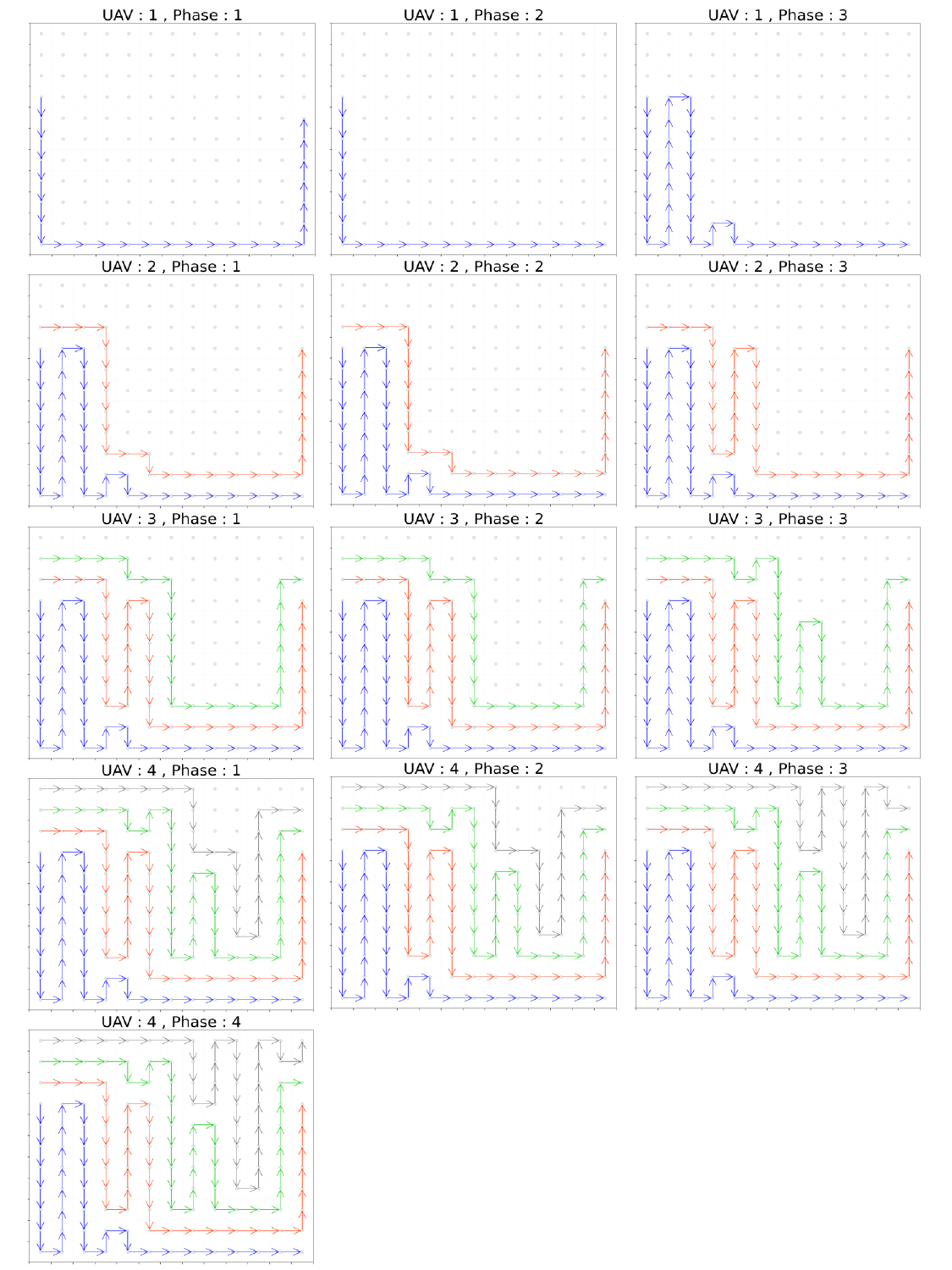}
    \caption{Demonstration of the algorithm's phases for 4 UAVs in $13\times11$ search area (Case no.\hspace{0.2em}2 in Table \ref{tab: mid case result}.)}
    \label{fig:case2midsize}
\end{figure}
\begin{figure}[H]
    \centering
    \includegraphics[scale=0.65]{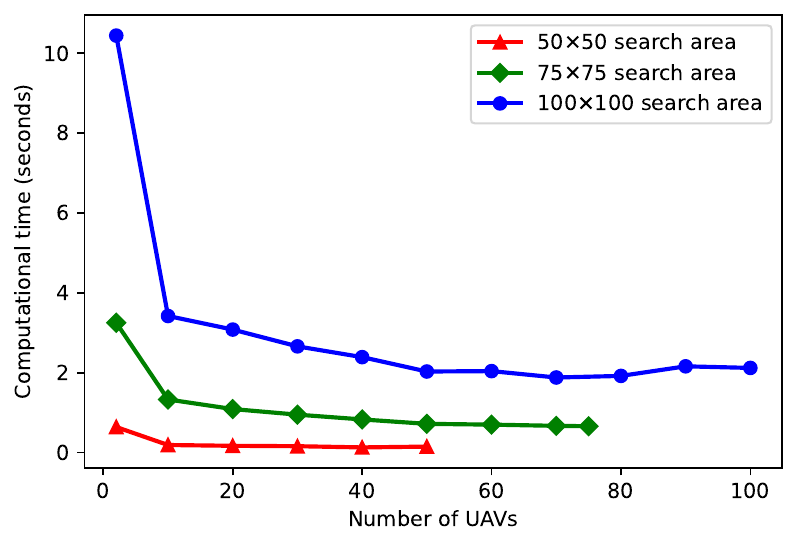}
    \caption{Impact of the number of UAVs on the computational time.}
    \label{fig:Q-sensivity}
\end{figure}

\subsection{Large-sized cases and sensitivity analysis}
In this section, the results of the algorithm's performance on large-sized cases, which range from 1000 to 10,000 cells, are presented in Table \ref{tab: large case result}. It is important to note that the results in this section carry greater significance as they reflect the algorithm's efficiency in handling cases that closely resemble real-world scenarios. According to the findings, the algorithm demonstrates a very good performance in large-sized cases, and generates solutions achieving an optimality gap of $0.00\%$, all within a reasonable timeframe. Also, as expected, all solutions belong to the set $\{\text{LB},\text{LB}+T_p\}$.

In the following step, we examine the impact of varying the number of UAVs on computational time, as depicted in Fig. \ref{fig:Q-sensivity}. This analysis shows some aspects of the algorithm's behavior. Firstly, the maximum computational time is observed when there are only two UAVs. Secondly, while the addition of UAVs does increase problem complexity, it initially accelerates algorithm performance. Thirdly, after a certain point, altering the number of UAVs exhibits no substantial influence on computational time.

Figure \ref{fig:min-UAV&dim}b illustrates how varying dimensions within a rectangular search area affect computational time. Seven cases with different dimensions are presented, each containing 10,000 cells. The results demonstrate that the algorithm requires less computational time in search areas with smaller length and greater width. Furthermore, it is reaffirmed that across all cases, employing two UAVs consistently demands more computational time. Considering the battery limitation of the UAVs, each operation time has an upper bound. Given one hour as the upper bound of the operation time, Figure \ref{fig:min-UAV&dim}a shows the minimum number of UAVs required to cover the cases in Table \ref{tab: large case result} in less than one hour. In this regard, it is noteworthy to mention that the LB formula can provide a very good estimation for calculating the minimum number of UAVs required for an operation with a determined time.
\begin{table}[h]
\caption{Results for large-sized test cases} \label{tab: large case result} 
\begin{tabular*}{\textwidth}{@{\extracolsep\fill}lccccccc}
\toprule%
Case no.&Size of the&Number&LB& \multicolumn{2}{@{}c@{}}{NOPP} &Optimality & Absolute \\\cmidrule{5-6}%
 & $n\times m$ search area & of UAVs & & $Z_{\mathit{NOPP}}$ & Time(s) & gap($\%$)& gap(s)\\
\midrule
1&$50\times 50$&2&6388.00&6388.00&0.81&0.00&0.00\\
2&$50\times 75$&2&9613.00&9613.00&1.71&0.00&0.00\\
3&$75\times 50$&2&9584.00&9584.00&2.35&0.00&0.00\\
4&$75\times 75$&2&14424.08&14429.24&2.92&$\approx 0.00$&5.16\\
5&$75\times 100$&2&19259.00&19264.16&5.37&$\approx 0.00$&5.16\\
6&$100\times 75$&2&19230.00&19230.00&6.34&0.00&0.00\\
7&$100\times 100$&2&25680.0&25680.0&10.44&0.00&0.00\\
\botrule
\end{tabular*}
\footnotetext{LB is obtained from LB formula in Sect. \ref{sec:LBproof}. Optimality gap = $\frac{|LB-Z_{\mathit{NOPP}}|}{LB}\times 100$, and absolute gap = $|LB-Z_{\mathit{NOPP}}|$.}
\end{table}

\begin{figure}[h]
    \centering
    \subfigure(a)
  {\includegraphics[width=0.45\textwidth]
    {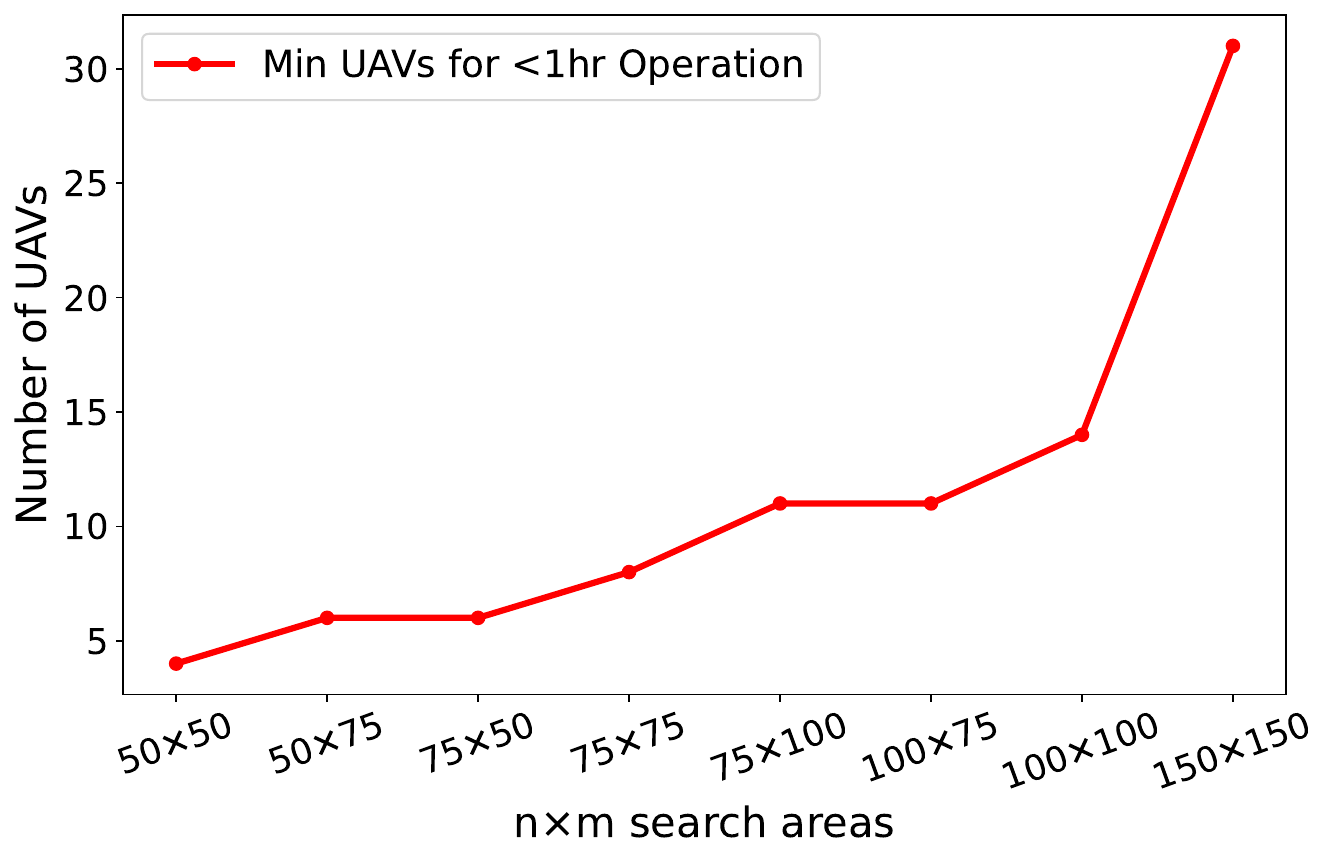}} 
    \hspace{0.2 cm}
    \subfigure(b)
{\includegraphics[width=0.45\textwidth]{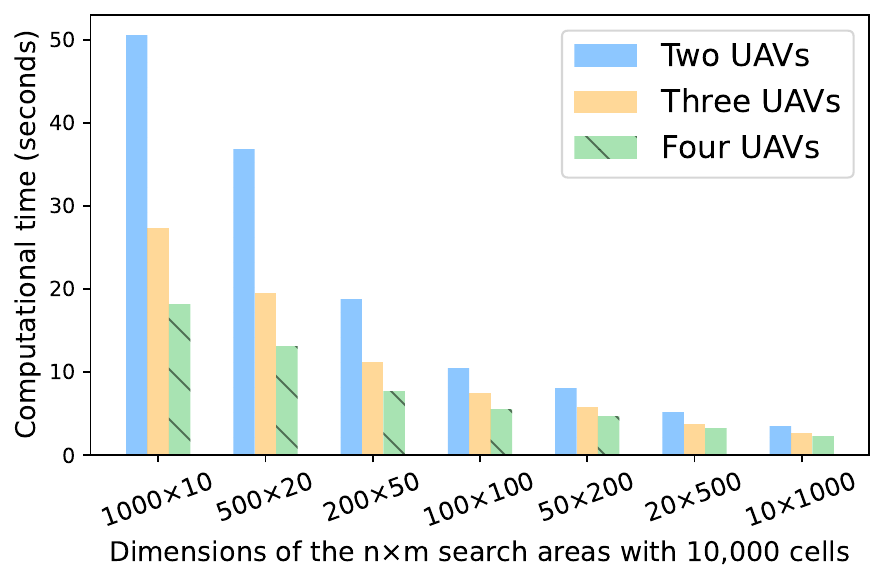}} 
    \caption{(a) Minimum UAV number for operations under one hour; (b) Impact of different search area dimensions on computational time.}
    \label{fig:min-UAV&dim}
\end{figure}

\section{Conclusion and Future Work} \label{sec:conclusion}
Using UAVs for search and rescue comes with its own challenges such as the risk of collision and environmental factors. To deal with these challenges, we developed a path-planning algorithm to enable a fleet of UAVs to efficiently search a target area in windy conditions. We first proposed a mixed-integer programming model to formulate the problem. Then, to design a practical solution approach, we investigated a special case of the problem where the search area is a rectangular grid. In this context, we presented a mathematically validated formula for calculating the problem's lower bound. Next, we proposed the NOPP algorithm that consistently yields feasible solutions, achieving either the lower bound or approaching it closely.

We conducted a series of experiments encompassing scenarios of varying search area sizes: small, medium, and large cases. The outcomes of these experiments consistently reveal the algorithm's remarkable proficiency in promptly generating feasible solutions, even when confronted with a large case containing up to 10,000 cells. Also, as indicated by the results, augmenting the problem's complexity by adding more UAVs not only fails to yield a significant impact on the algorithm's speed but, in some cases, can actually lead to an increase in its speed. Furthermore, the results demonstrate that modifying the dimensions of a search area while maintaining a constant cell count significantly influences the computational time. In summation, our findings validate the efficacy of the proposed algorithm as a dependable tool for search and rescue operations. Its ability to efficiently devise feasible, optimal, or near-optimal solutions within a reasonable timeframe, particularly in the context of large-scale scenarios, underscores its potential value to search and rescue teams.

Several aspects of the present work can be extended. First, the mathematical framework devised for formulating the lower-bound calculation of the problem can be applied in scenarios involving heterogeneous UAV fleets with unequal speeds. Second,  polygonal shapes can be considered as an alternative to rectangular shapes. Our primary focus in this research has been on search and rescue, leading us to concentrate solely on operation time while neglecting other aspects. Future studies can broaden their scope to include factors like battery consumption and landing location, making them more applicable to other applications.

\appendix
\section*{Appendix A: Complete proof of feasibility and near-optimality for the NOPP algorithm}
We start by introducing a foundational framework that forms the basis for the subsequent propositions of this section. We define $C_r^{k}$ as a class of grid-based shapes, all of which share the following three properties:
\begin{enumerate}
  \item[1)]
  All shapes possess $r$ columns, and at least one of the columns has $k$ cells (the distance between the baseline and the top line is k cells, as shown in Fig. \ref{fig:G-shape}a).
  
  \item[2)] 
  Within each column, cells are required to initiate from the top line and remain connected, no disjoint.

  \item[3)] 
  Every column contains a minimum of one and a maximum of $k$ cells.

\end{enumerate}
\begin{figure}[H]
    \centering
    \subfigure(a){\includegraphics[width=0.45\textwidth]{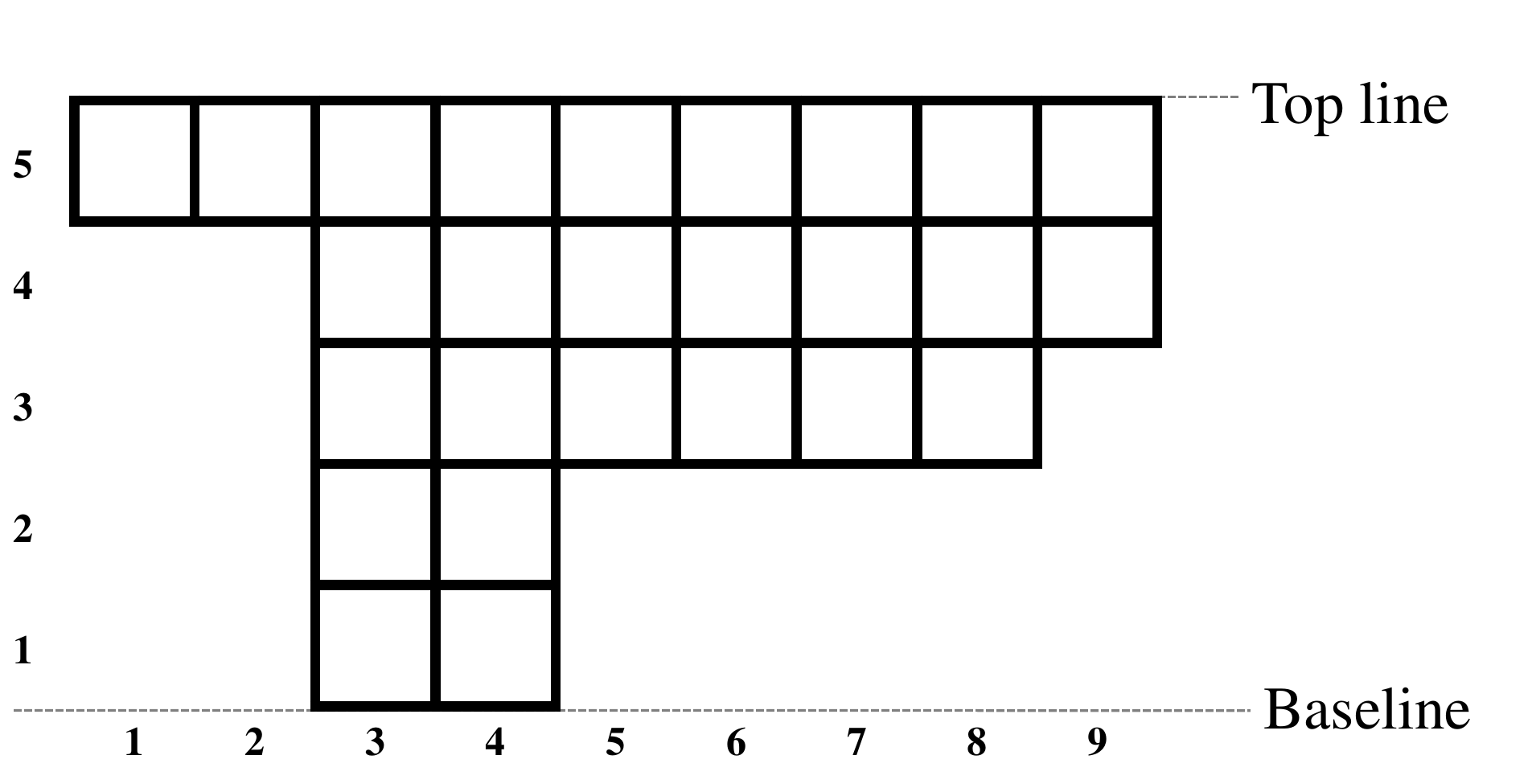}} 
    \subfigure(b){\includegraphics[width=0.45\textwidth]{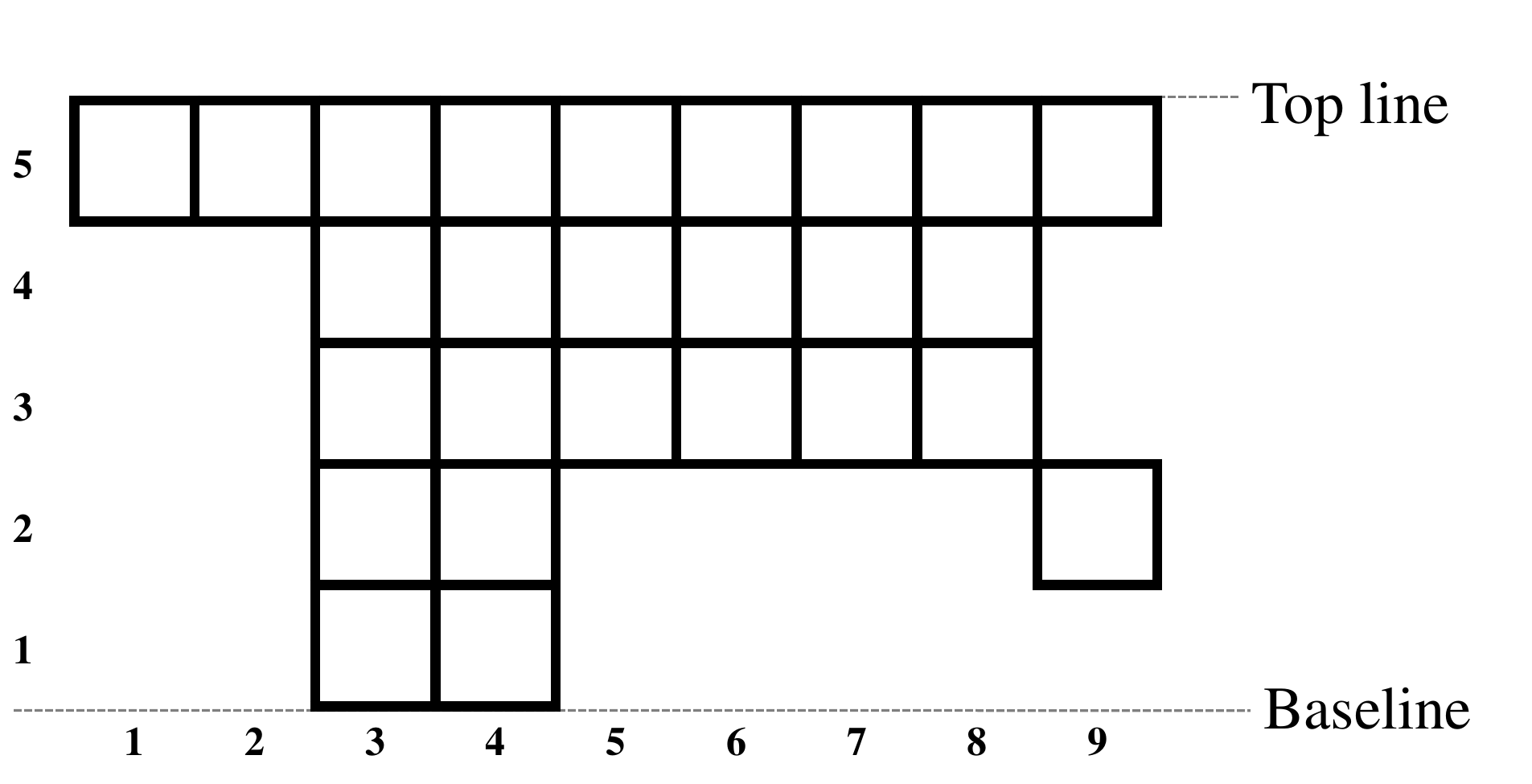}} 
    \subfigure(c){\includegraphics[width=0.45\textwidth]{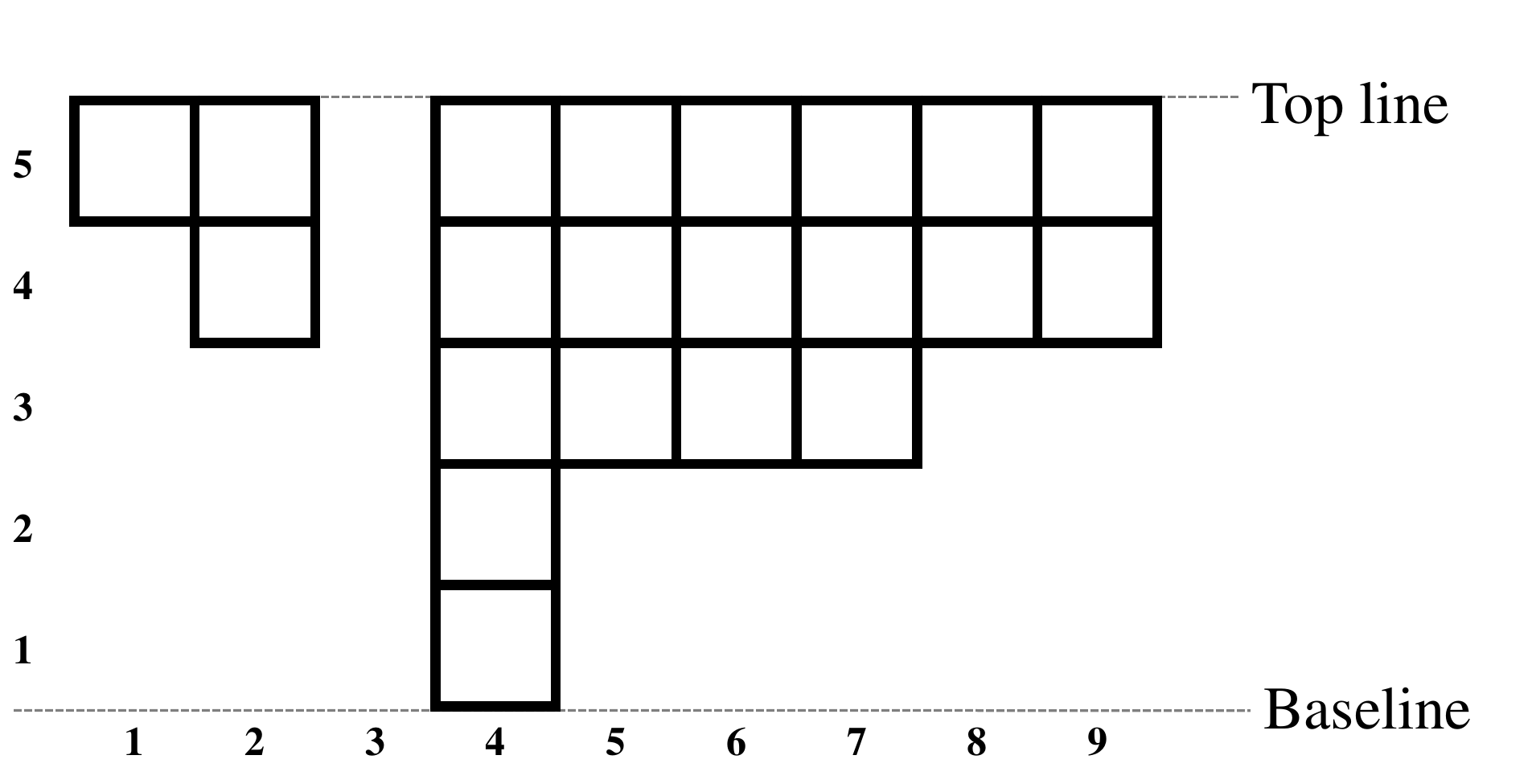}}
    \subfigure(d){\includegraphics[width=0.45\textwidth]{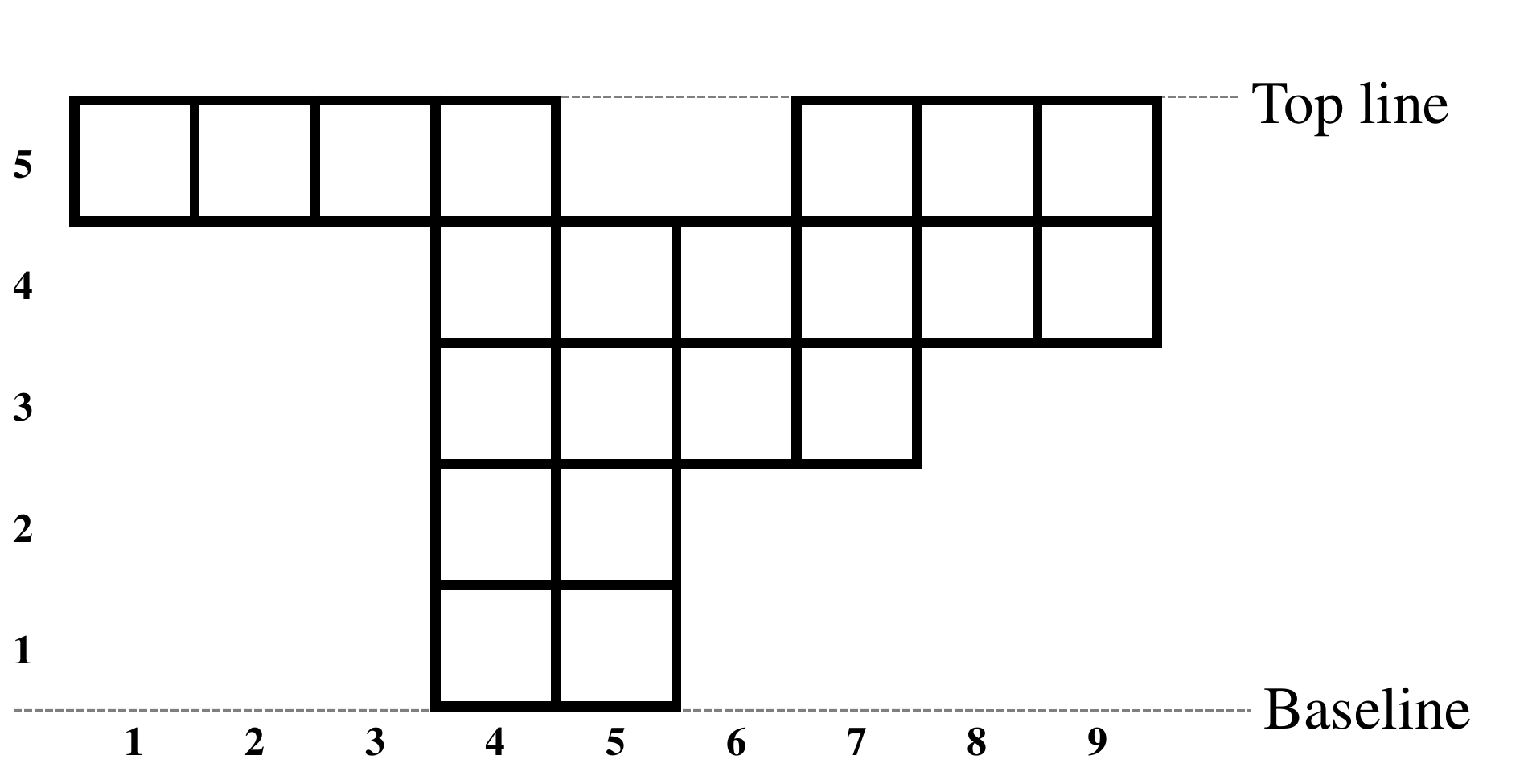}}
    \caption{Valid and invalid examples of $C_9^5$ shapes: (a) valid; (b) invalid because there is a disjoint in column 9; (c) invalid because column 3 does not have even one cell; (d) invalid because cells in columns 5 and 6 do not initiate from the top line.}
    \label{fig:G-shape}
\end{figure}
In addition,
the notations utilized in this section are presented in Table \ref{tab:algorithm proof notation} for reference.

\begin{table}[h]
\caption{Algorithm's feasibility proof notations} \label{tab:algorithm proof notation} 
\begin{tabular*}{\textwidth}{@{}ll@{}}
\toprule
Notations & \\
\midrule
$\Bar{I}$ & Set of indices pertaining to the columns of the $C_r^k$ shape, $\{1,2,3,...,r\}$\\
$p^i$ &
Set of cells in column $i$, $\{(i,k),(i,k-1),\dots\}$ for $i \in \bar{I}$, is arranged sequentially \\
& and referred to as ``pit $i$"\\
$v_t$ & Subset of $\bar{I}$ refered to as valley of length $t$, $\{i_1$,\dots,$i_t\}$ for $i_k\in \bar{I} \ (k=1,\dots,t)$, where:\\

&\ \ \  $\bullet$ Indices are sequentially ascending \\ & 
 \ \ \ $\bullet$ If $t>1$, then $\forall i,j \in v_t and\  i \neq j, |p^i|=|p^j|$\\
 &\ \ \ $\bullet$ If  $i_1\neq 1$,then $|p^{i_1-1}|\neq |p^{i_1}|$ \\
 &\ \ \ $\bullet$ If $i_t\neq n$, then $|p^{i_t}|\neq |p^{i_t+1}|$ \\
$d(v_t)$ & Depth of an t-length valley calculated as $\forall i \in v_t, d(v_t)=|p^i|$\\
$G$ & A partition of $\bar{I}$ such that:\\
&\ \ \ $\bullet$ All its members are valleys\\
& \ \ \ $\bullet$ Elements of each valley are greater than any element of the preceding valley\\
\botrule
\end{tabular*}

\end{table}
Continuing further, according to the definitions in Table \ref{tab:algorithm proof notation}, we introduce essential terminology employed within the propositions of this section. To begin, for a $C_r^k$ shape, we establish the term \emph{first valley} to denote the valley (as defined in Table \ref{tab:algorithm proof notation} as $v_t$) containing element 1 as well as the term \emph{final valley} to refer to the valley encompassing the element $r$, and also, the term \emph{middle valley}  is designated for other valleys. It should be mentioned that if a valley has both 1 and $r$, it is considered as the first valley.  Furthermore, we introduce the concept of \emph{feasibility feature (FF)}, wherein all valleys in a $C_r^k$ shape are constrained to possess even cardinalities (the cardinality of a valley is equal to its length) except the first and final valleys. The notation ``FF-G" is adopted to signify the set $G$ (as defined in Table \ref{tab:algorithm proof notation}) of a $C_r^k$ shape endowed with the feasibility feature. Moreover, if a $C_r^k$ shape is representable in the form of an FF-G set, it is designated as an \emph{FF-G shape}. For a clearer understanding of defined notations and definitions, consider Figure \ref{fig:G-shape}a as an example, depicting an FF-G shape with its corresponding $G$ set equal to $\{\{1,2\},\{3,4\},\{5,6,7,8\},\{9\}\}$. The first valley is $v_1=\{1,2\}$ and to find its depth, we need to check the pits (as defined in Table \ref{tab:algorithm proof notation} as $p^i$) with indices 1 and 2. For the pit with index 1, we have $p^1=\{(1,5)\}$, and for the pit with index 2, $p^2=\{(2,5)\}$. So, the depth of the first valley is $d(v_1)=|p^1|=|p^2|=1$. The middle valleys are $v_2=\{3,4\}$ and $v_3=\{5,6,7,8\}$. For $v_2$, we have to consider pits with index 3 and 4, So $p^3=\{(3,5),(3,4),(3,3),(3,2),(3,1)\}$, and $p^4=\{(4,5),(4,4),(4,3),(4,2)(4,1)\}$; then we have, $d(v_2)=|p^3|=|p^4|=5$. For $v_3$, pits with indices $5,6,7,8$ should be considered. So, we have $p^5=\{(5,5),(5,4),(5,3)\},\ p^6=\{(6,5),(6,4),(6,3)\},\ p^7=\{(7,5),(7,4),(7,3)\},\ \text{and} \ p^8=\{(8,5),(8,4),(8,3)\}$ as well as $d(v_3)=|p^5|=|p^6|=|p^7|=|p^8|=3$. The finall valley is $v_4=\{9\}$ in which $p^9=\{(9,5),(9,4)\}$ and $d(v_4)=|p^9|=2$. It should be added that the length of valleys $v_1$ (the first valley), $v_2$, $v_3$, and $v_4$ (the final valley) are $2,2,4$, and $1$ respectively. Since the length of the middle valleys ($v_2$ and $v_3$) is even, we have an FF-G shape. Now, the following propositions can be stated.

\begin{prop}\label{prop:FF-G}
There is a feasible path to cover all cells within an FF-G shape.
\end{prop}

\begin{proof}
To prove the proposition, it is necessary to consider the following cases for the valleys of set $G$:
\begin{enumerate}
  \item[I)]
  All valleys are even-length:\\ 
  Given that the length of each valley is an even number, they can be divided into one or more pairs of pit indices. For each pair of pit indices, consider the pair of pits corresponding to it. To cover each pair of pits, we can start from the first cell of the first pit (pit with the lower index) and use $D$ moves to reach the end cell. Following this, with a single $S$ move, we can transit to the final cell of the next pit, and subsequently cover all cells in this pit using $U$ moves. Finally, by employing $S$ moves, we can establish connectivity across all pairs.
  \item[II)] 
  All valleys are even-length except for the first valley:\\
  Initially, omit 1 (index of the first column) from the first valley. As a result, all valleys become even-length, thereby allowing us to use the pattern of case one. Finally, we can cover the cells in $p^1$ by starting from the last cell and using $U$ moves, and also it can be connected to the next pit by an $S$ move.
   \item[III)]
   All valleys are even-length except for the final valley:\\Initially, disregard $r$ (index of the last column) from the last valley. Consequently, all valleys become even-length, enabling us to apply the pattern of case one. At the end, cells in $p^r$ can be covered by starting from the first cell and employing $D$ moves. It must be mentioned that an $S$ move makes a connection between $p^{r-1}$ and $p^r$.
   \item[IV)]
   All valleys are even-length except for the first and final valleys:\\ This case can be effectively resolved by combining cases two and three.
\end{enumerate}
In addition, it is crucial to highlight that when the depth of a valley equals one, the associated pits only need $S$ moves to be covered and connected to other pits. Therefore, it is consistently achievable to create a feasible path for covering an FF-G shape.
\end{proof}

Before proceeding, we need to define two types of FF-G shapes that will be used in propositions \ref{prop:feasibility2} and \ref{prop:feasibility3}. These types are:
\begin{enumerate}
    \item [$\bullet$] FF-G1: This refers to FF-G shapes with just one valley. For example, a rectangle is an FF-G1 shape.
    \item [$\bullet$] FF-G2: This refers to FF-G shapes where the first valley has a depth of one, and the depth of middle valleys is greater than one. For instance, Figure \ref{fig:G-shape}a shows an FF-G2 shape.
\end{enumerate}
Also, since the feasibility and optimality of the proposed algorithm are obvious for $q=m$, for the next propositions, we suppose that $q<m$ and $q\le n$.\\
   
\begin{prop} \label{prop:feasibility2}
Upon completion of Phases 1, 2 and 3 of the NOPP algorithm which result in a path for the $i^{\text{th}}$ UAV ($i=1,2,\dots,q-1$), available cells for the $(i+1)^{\text{th}}$ UAV form an FF-G2 shape.  
\end{prop}

\begin{proof}
    In the beginning, let us state that the available cells for the $i^{\text{th}}$ UAV ($i=1,2,\dots,q$) are those uncovered cells that satisfy the Y-coordinate condition (as
explained in Sect. \ref{sec:phase1}). In this proof, we employ $g_i$ (for $i = 1, ..., |G|$) to denote the $i^{th}$ valley (member) of the set G, and we also use $g_{i}^{\prime}$ to represent the $i^{th}$ valley after completing Phases 1 and 2. Additionally, we assume that the depth of a valley can be temporarily zero. 

    First, we consider the FF-G1 shape that is related to the first UAV. Based on the algorithm design, the first UAV starts from the cell $(1,m-q+1)$. Therefore, we can say that the available cells for the first UAV make a $C_n^{m-q+1}$ shape. Since this shape has only one valley ($G=\{\{1,2,\dots,n\}\}$), it is an FF-G1 shape.

    After Phases 1 and 2, the two following cases can occur for the single valley of the shape ($g_1$):
    \begin{enumerate}
        \item [I)] It is broken down into two valleys  ($g_\alpha$ and $g_\beta$) that are:
        \begin{enumerate}
            \item [$\bullet$] $g_\alpha=\{1\}$, and $d(g_\alpha)=0$
            \item[$\bullet$] $g_\beta=\{2,\dots,n\}$, and 
        $d(g_\beta)=d(g_1)-1$
        \end{enumerate}
        \item [II)] It is broken down into three valleys ($g_\alpha$, $g_\beta$, and $g_\gamma$) that are:
        \begin{enumerate}
              \item [$\bullet$] $g_\alpha=\{1\}$, and $d(g_\alpha)=0$
            \item[$\bullet$] $g_\beta=\{2,\dots,n-1\}$, and $d(g_\beta)=d(g_1)-1$
             \item[$\bullet$] $g_\gamma=\{n\}$, and
         $d(g_\gamma)=d(g_1)-k$ (for $k=2,\dots,d(g_1)$)
        \end{enumerate}
    \end{enumerate}
    Phase 2 guarantees that $|g_\beta|$ is even in the second case. Also, since Phase 3 considers a pair of pits in each valley from the lowest index, the length of the middle valleys is even after implementing this phase. We know that, for the next UAV (the second UAV), the starting point is $(1,m-q+2)$, so the top line for this UAV moves up one cell above the top line of the previous UAV (the first UAV). Therefore, we can add one unit to the depth of the valleys and finally, we have an FF-G2 shape for the second UAV.

    Now we have an FF-G2 shape. Initially, disregarding the first and final valleys, the following cases demonstrate how the middle valleys are modified when Phases 1 and 2 are applied. Therefore, these cases for $g_i$ (for $i \in \{2,\dots,|G|-1\}$) are:
    \begin{enumerate}
        \item[I)] If $d(g_{i-1})<d(g_i)<d(g_{i+1})$, then $|g_{i}^{\prime}|=|g_i|$ and $d(g_{i}^{\prime})=d(g_i)-1$
        
        \item[II)] If $d(g_{i-1}) < d(g_i)$ and $d(g_{i}) > d(g_{i+1})$, then $|g_{i}^{\prime}|=|g_i|-2$ and $d(g_{i}^{\prime})=d(g_i)-1$
        \item[III)] If $d(g_{i-1})>d(g_i)>d(g_{i+1})$, then $|g_{i}^{\prime}|=|g_i|$ and $d(g_{i}^{\prime})=d(g_i)-1$
        \item[IV)] If $d(g_{i-1}) > d(g_i)$ and $d(g_{i}) < d(g_{i+1})$, then $|g_{i}^{\prime}|=|g_i|+2$ and $d(g_{i}^{\prime})=d(g_i)-1$
    \end{enumerate}
    For the next step, we consider the first valley. Since the first valley is a one-depth valley, there is only one case:
    \begin{enumerate}
        \item[I)] $|g_{1}^{\prime}|=|g_{1}|+1$ and $d(g_{1}^{\prime})=d(g_1)-1=0$
    \end{enumerate}
    In the following, the final valley, $g_{|G|}$, should be evaluated in some cases. After implementing Phases 1 and 2, there are two cases for the final valley:
    \begin{enumerate}
        \item [I)] It is broken into separate valleys ($g_a$ and $g_b$). So, we have the following sub-cases:
        \begin{enumerate}
            \item [1)] if d$(g_{|G|}) < $d$(g_{|G|-1})$ then we have:
            \begin{enumerate}
                \item [$\bullet$] $|g_a|=|g_{|G|}|$, d($g_a$)=d($g_{|G|}$)-1
                \item [$\bullet$] $|g_b|=1$,
              $g_b=\{n\}$, $d(g_b$)=$d(g_{|G|})-k$, for $k=2,3,\dots,d(g_{|G|})$
            \end{enumerate}
             \item [2)] if d$(g_{|G|}) > $d$(g_{|G|-1})$ then we have:
             \begin{enumerate}
                 \item [$\bullet$]
                 $|g_a|=|g_{|G|}|-2$, d($g_a$)=d($g_{|G|}$)-1
                 \item[$\bullet$] 
                 $|g_b|=1$, $g_b=\{n\}$, $d(g_b$)=$d(g_{|G|})-k$, for $k=2,3,\dots,d(g_{|G|})$
             \end{enumerate}
        \end{enumerate}
        Phase 2 guarantees that $|g_a|$ is always even.
        \item [II)] It is not broken into separate valleys, then we have two sub-cases for $g_{|G|}^\prime$:
        \begin{enumerate}
      
            \item [1)] if d$(g_{|G|}) < $d$(g_{|G|-1})$then we have:
            \begin{enumerate}
                \item [$\bullet$]
                $|g_{|G|}^\prime|=|g_{|G|}|+1$, d($g_{|G|}^\prime$)=d($g_{|G|}$)-1
            \end{enumerate}
             \item [2)] if d$(g_{|G|}) > $d$(g_{|G|-1})$then we have:
             \begin{enumerate}
                 \item [$\bullet$]
                 $|g_{|G|}^\prime|=|g_{|G|}|-1$, d($g_{|G|}^\prime$)=d($g_{|G|}$)-1
             \end{enumerate}
           
        \end{enumerate}
    \end{enumerate}
After these changes, it is evident that the middle valleys have even lengths following the execution of Phases 1 and 2. As Phase 3 of the algorithm operates on pairs of pits, the parity of the length of middle valleys remains unchanged. For the next UAV (the third UAV), we shift the top line up by one cell, so all the depth of valleys increases by one unit. Consequently, the depth of the first valley is equal to one, and the remaining shape for the third UAV is an FF-G2 shape. Finally, we can see that this pattern will be continued for the next UAVs, and therefore, available cells for $i^{th}$ UAV ($i=2,\dots,q$) make an FF-G2 shape.
    
\end{proof}
Now, we can prove the feasibility of the generated solutions by the proposed algorithm.
\begin{prop}\label{prop:feasibility3}
    The NOPP algorithm generates feasible solutions.
\end{prop}

\begin{proof}
    According to Proposition \ref{prop:feasibility2}, the last remaining shape for the final UAV is an FF-G2 shape. Since an FF-G2 shape is a type of FF-G shape, it can be covered by the four phases of the proposed algorithm, as proved in Proposition \ref{prop:FF-G}.
\end{proof}
 To better understand Proposition \ref{prop:feasibility3}, Figure \ref{fig:appendixB-final} illustrates an example featuring five UAVs within a 13 × 13 search area. The proposed algorithm starts generating a path for the first UAV in a rectangle, forming an FF-G1 shape. According to Proposition \ref{prop:feasibility2}, subsequent shapes for the following UAVs are designated as FF-G2 shapes. Consequently, an FF-G2 shape remains for the fifth and last UAV, which can be covered within four phases of the algorithm.

\begin{figure}[h]
    \centering
    \includegraphics[width=0.9\textwidth]{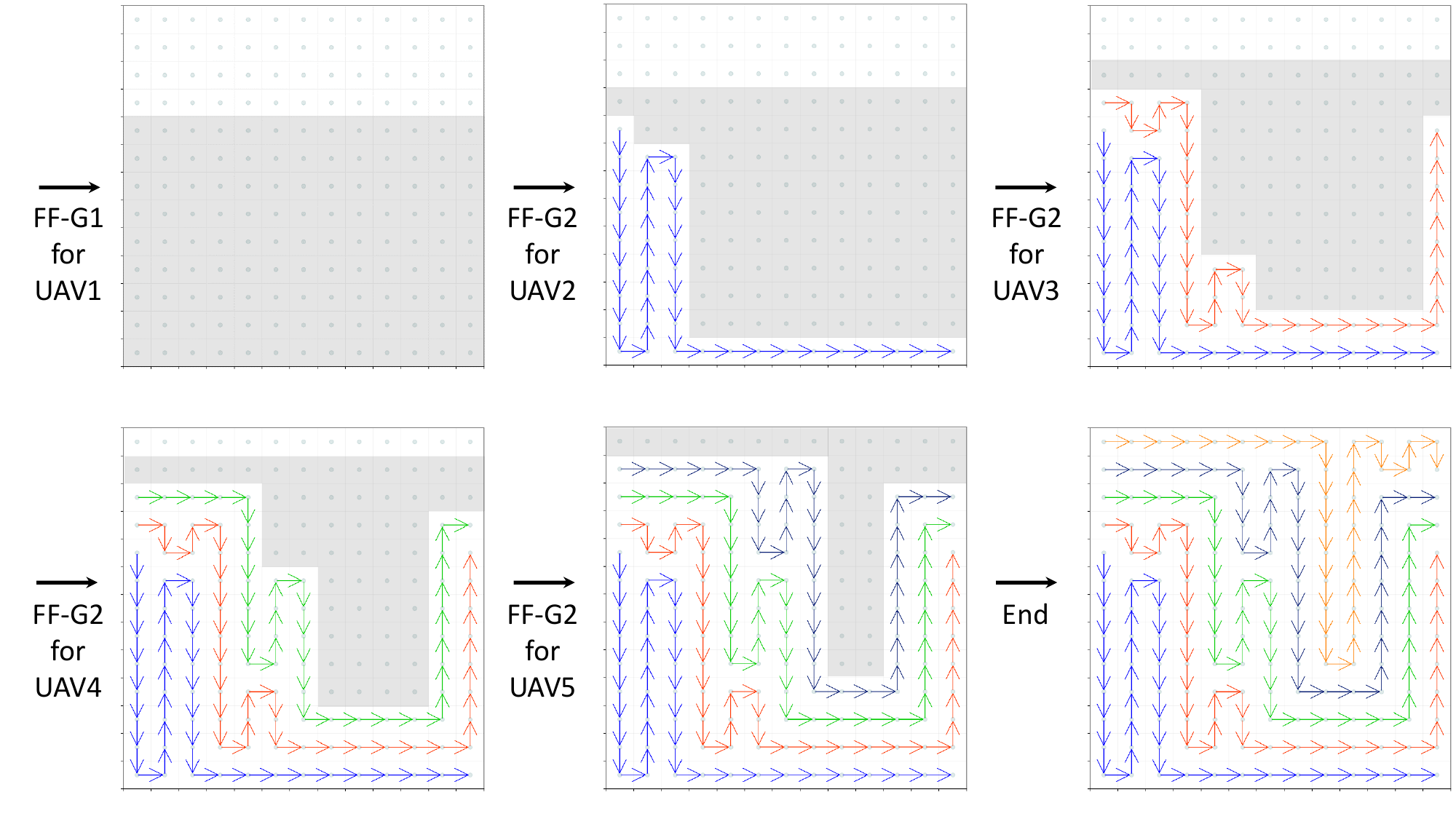}
    \caption{Example of the feasibility of the algorithm's solutions, as proved in Proposition \ref{prop:feasibility3} ($n=13$, $m=13$, $q=5$). All gray shapes are $C^9_{13}$ shapes in this example.}
   \label{fig:appendixB-final}
\end{figure}
By the next proposition, we prove the lower bound and upper bound of the algorithm's solution.

\begin{prop}\label{prop:solution-space}
    The solutions generated by the NOPP algorithm have objective values in the set $\{LB, LB+T_p\}$.
\end{prop}
\begin{proof}
Let $A = \lceil \frac{nm}{q}\rceil$ denote the upper bound on the allocation of cells to a UAV required to achieve the LB (as used in Sect. \ref{sec:initial setting}). Based on the proposed algorithm's design, three following cases may occur.
\begin{enumerate}
    \item[$1)$] The first $q-1$ UAVs cover either $A$ or $A + 1$ (by using an extra $P$
move as explained in Sect. \ref{sec:phase3}) cells. As a result, the last UAV is left with fewer than $A + 1$ cells to cover. Consequently, the operation time can be either $LB$ or $LB + T_p$.
    \item[$2)$] In this case, we have a UAV, denoted by $i^{th}$ UAV ($1<i<q$), that covers less than $A$ cells.  This situation arises when the number of cells that meet the Y-coordinate condition (as detailed in Sect. \ref{sec:phase1}) remaining for the $i^{th}$ UAV is less than $A$, and the $i^{th}$ UAV cover these cells. In this case, UAVs numbered from 1 to $i-1$ have covered either $A$ or $A+1$ cells, while UAVs numbered from $i+1$ to $q$ are required to cover $n$ cells. Since $n$ is consistently less than $\lceil \frac{nm}{q} \rceil $, the resulting operation time is either $LB$ or $LB+T_p$.
    
    \item [$3)$] This case is a sub-state of the previous case. According to the second case, the number of cells left for $i^{th}$ UAV ($1<i<q$) is less than $A$. However, in this case, the $i^{th}$ UAV covers all of those remaining cells except some of them whose X-coordinate is $n$. This situation arises due to the algorithm's phase 3, which considers a pair of pits at each step. In this case, UAVs from 1 to $i - 1$ have covered either $A$ or $A + 1$ cells, and UAV $i + 1$ has a maximum of $n + m - q$ cells to cover. Since $m-q \le (m-q)\frac{n}{q}$, this leads to $n+m-q \le \frac{nm}{q}\le \lceil \frac{nm}{q} \rceil$. If $i\le q-2$, UAVs from $i + 2$ to $q$ should cover $n$ cells. As a result, the operation time is either $LB$ or $LB+T_p$.
\end{enumerate}
\end{proof}

\section*{Acknowledgements}
The research is supported by the National Science Foundation (NSF) under grant CMMI 1944068.

\printbibliography

\end{document}